\documentclass{article}

\usepackage{microtype}
\usepackage{graphicx}
\usepackage{subfigure}
\usepackage{booktabs} 

\usepackage[utf8]{inputenc}
\usepackage{xcolor}
\usepackage{amsmath}
\usepackage{amsthm}
\usepackage{amssymb}
\usepackage{comment}
\usepackage{placeins}
\usepackage{caption}
\usepackage{multirow}
\usepackage{fullpage}

\usepackage{mathrsfs}
\usepackage{enumerate}

\usepackage{hyperref}
\usepackage{dsfont}

\usepackage{algorithm}
\usepackage{algorithmic}

\newcommand{\E}{\mathbb{E}}
\renewcommand{\P}{\mathbb{P}}
\newcommand{\one}{\mathbb{I}}

\newcommand{\cH}{\mathcal{H}}

\newcommand{\cE}{\mathcal{E}}
\newcommand{\cA}{\mathcal{A}}

\newcommand{\cS}{\mathcal{S}}
\newcommand{\cB}{\mathcal{B}}

\newcommand{\cN}{\mathcal{N}}

\newcommand{\cJ}{\mathcal{J}}

\newcommand{\algoname}{\textsc{ArmSwitch}\;}
\newcommand{\bettername}{\textsc{Elim}\xspace}
\newcommand{\better}[3]{\bettername_{#3}(#1,#2)}

\newcommand{\GOOD}{\text{GOOD}\xspace} 
\newcommand{\BAD}{\text{BAD}\xspace} 
\newcommand{\CONST}{\textsc{Const}\xspace} 
\newcommand{\EXP}{\textsc{EXP}} 
\newcommand{\B}{B} 

\DeclareMathOperator*{\argmax}{argmax}

\usepackage{mathtools}
\usepackage{xspace}

\newtheorem{theorem}{Theorem}[section]
\newtheorem{corollary}[theorem]{Corollary}
\newtheorem{lemma}[theorem]{Lemma}

\usepackage[round,sectionbib]{natbib}
\usepackage{xspace}
\usepackage[textsize=tiny, disable]{todonotes}

\allowdisplaybreaks

\title{\LARGE \bf
A New Look at Dynamic Regret for Non-Stationary Stochastic Bandits
}
\author{Yasin Abbasi-Yadkori \\ 
DeepMind\\
{\small\tt yadkori@deepmind.com} 
\and 
Andr\'{a}s Gy\"{o}rgy \\ 
DeepMind\\
{\small\tt agyorgy@deepmind.com} 
\and 
Nevena Lazi\'{c} \\ 
DeepMind\\
{\small\tt nevena@deepmind.com} 
}
\date{}                     

\begin{document}

\maketitle

\begin{abstract}
We study the non-stationary stochastic multi-armed bandit problem, where the reward statistics of each arm may change several times during the course of learning. The performance of a learning algorithm is evaluated in terms of their dynamic regret, which is defined as the difference between the expected cumulative reward of an agent choosing the optimal arm in every time step and the cumulative reward of the learning algorithm. One way to measure the hardness of such environments is to consider how many times the identity of the optimal arm can change. We propose a method that achieves, in $K$-armed bandit problems, a near-optimal $\widetilde O(\sqrt{K N(S+1)})$ dynamic regret, where $N$ is the time horizon of the problem and $S$ is the number of times the identity of the optimal arm changes, without prior knowledge of $S$. Previous works for this problem obtain regret bounds that scale with the number of changes (or the amount of change) in the reward functions, which can be much larger, or assume prior knowledge of $S$ to achieve similar bounds.
\end{abstract}

\section{Introduction}

The multi-armed bandit (MAB) problem is the canonical problem to study the exploration-exploitation dilemma. At each time step $n\in \{1, \ldots, N\}$, the learner selects an arm (also called action) $a_n \in \{1, .., K\}$ and receives a reward $r_n$ generated from an unknown distribution which may depend on both the time step and the action. The learner's goal is to maximize the sum of the rewards. 
In the standard stochastic MAB problem, the reward distribution for each arm is assumed to be stationary, and algorithms are evaluated based on their expected \emph{regret}, which is the difference between the expected rewards obtained by the algorithm and the best fixed arm in hindsight. 

In this work, we consider the MAB problem with reward distributions that are non-stationary and can change several times during the course of learning. 
We evaluate learning algorithms in terms of their dynamic regret, which is the difference between the cumulative expected rewards obtained by the best non-stationary policy selecting the optimal arm in every time step and that of the learning algorithm.
MAB problems with non-stationary reward distributions have been studied extensively in literature. This includes a variety of settings, including adversarial rewards, limited total variation of change \citep{besbes2014stochastic}, limited number of switches \citep{auer2002nonstochastic}, as well as additional assumptions on the process generating the changes \citep{ortner2014regret,slivkins2008adapting}. 

In general, the achievable dynamic regret depends on the assumptions made about the reward process.
When the reward distribution changes at most $M$ times, known as the switching bandit problem \citep{garivier2011upper}, the EXP3.S algorithm of \cite{auer2002nonstochastic} can achieve $O(\sqrt{K(M+1)N \log(KN)})$ regret, when its tuning depends on $M$, which therefore needs to be known in advance. This result is known to be minimax optimal up to the logarithmic factors. Several other algorithms can be tuned to achieve bounds that are optimal in $N$ and $M$ given the prior knowledge of $M$, including the sliding-window UCB algorithm of \citet{garivier2011upper} and an elimination-based method by \citet{allesiardo2017non}.

The first algorithm to obtain near-optimal regret without knowing the number of changes $M$ is the \textsc{AdSwitch} algorithm of \citet{auer2018adaptively}. While the original version of \textsc{AdSwitch} is only applicable to the case of $K=2$, it was subsequently extended to a general $K$ by \citet{AGO-2019, CLLW-2019, ACGLLOW-2019}. The algorithm starts with an initial estimation stage that detects the current optimal arm, and then plays that arm while performing periodic exploration in order to detect changes. Another recent such algorithm is \textsc{Master} \citep{wei2021non}, which, specialized to the switching bandit problem, runs a baseline UCB1 algorithm \citep{auer2010ucb} at multiple time scales, and periodically resets if non-stationarity is detected.

While the number of times the reward distribution changes is often indeed related to the hardness of the problem, it can be a quite pessimistic measure of complexity: for example, a change in the reward distribution of a suboptimal arm that leaves it suboptimal, or a slight change in the reward of the optimal arm  so that it remains optimal should not really affect the performance of good learning algorithms.

To address this issue, in this paper we aim to bound the regret in terms of the number of changes in the identity of the optimal arm $S$, which can be much smaller than the number of reward changes $M$ considered in prior work (note that EXP3.S of \citealp{auer2002nonstochastic} can readily give a bound which scales with $S$ instead of $M$, but it requires prior knowledge of $S$, as discussed above). We propose a modified version of \textsc{AdSwitch}, called \algoname, which performs periodic exploration in order to detect a change in the optimal arm, rather than a change in the reward gap. This allows us to obtain a regret bound which scales as $\widetilde O(\sqrt{KN(S+1)})$, without the prior knowledge of $S$. \citet{SK-2022} obtain similar results in an independent and parallel work.

Similarly to \textsc{AdSwitch}, our algorithm is based on a phased elimination procedure with restarts. The original phased elimination procedure explores arms uniformly at random until it can be determined that an arm is not optimal, at which point the arm is eliminated and moved to the set of ``bad'' arms, BAD. The arms in the BAD set are occasionally explored to detect a potential change in the optimal arm. The algorithm restarts the phased elimination when all arms are moved to the BAD set. 

The phased elimination procedure uses confidence intervals to perform arm eliminations. In our setting, even though rewards can change in an arbitrary fashion, we can construct confidence intervals (based on the idea of importance weighting) for a notion of weighted total rewards using the random exploration property of phased elimination. Using such confidence intervals however introduces an issue: 
it is difficult to control the probability of detecting the optimal arm when the optimal arm is in the \BAD set and the number of active arms fluctuate in an arbitrary way, and this can lead to a regret bound that scales worse than the optimal $\sqrt{K}$ scaling in the number of actions. 

In order to resolve this issue, we use a non-uniform exploration in the phased elimination procedure that assigns a fixed $\frac{1}{K}$ sampling probability to an active \BAD arm,\footnote{By \BAD (\GOOD) arms, we mean the set of arms that are in the \BAD (\GOOD) set.} and use data from sampling arms with the fixed probability of $\frac{1}{K}$ in constructing a ``second type" confidence interval. Using such confidence intervals, we show a tighter control on the detection probability, eventually giving the desired overall $\widetilde O(\sqrt{K N (S+1)})$ regret bound. 


\subsection{Notation}
For any integer $K$, $[K]=\{1,\ldots,K\}$. For integers $n'<n$, we use $[n':n]$ to denote the set $\{n',n'+1,\ldots,n\}$ (we will often refer to such sets as intervals and use similarly half-open/open intervals not including the corresponding boundary points), and for any sequence $b_{n'},b_{n'+1},\ldots,b_{n}$, $b_{n':n}=\sum_{t=n'}^{n} b_t$. For a discrete set $I$ we denote its cardinality by $|I|$. For any two real numbers $x,y$, $x \vee y$ denotes their maximum $\max\{x,y\}$. For an event $\cE$, its complement is denoted by $\overline\cE$, and the indicator $\one\{\cE\}$ is 1 if $\cE$ holds and zero otherwise.

\section{Problem setting}

We consider multi-armed bandit problems with $K$ arms and a known time horizon $N$ (with $K,N\ge 2$). The interaction between the learning algorithm and the environment is as follows. At each time step $n$, the algorithm selects an arm $a \in [K]=\{1, \ldots, K\}$ and receives a reward $r_n \in [0,1]$ drawn according to an unknown distribution $D_n(a_n)$. 

Let $g_n(a)$ denote the mean of $D_n(a)$ for any $a \in [K]$ and $n \in [N]$, and let $a_n^* \in \argmax_{a \in [K]} g_n(a)$ be the optimal arm at time $n$. We assume $g_n(a) \in [0,1]$. Our assumptions on the rewards and their expectation is to simplify the presentation and can be replaced with a more general assumption with bounded expectation and sub-Gaussian noise. The sequence of reward distributions is chosen by an oblivious adversary before the start of the game. 
We evaluate the performance of algorithms in terms of the dynamic regret, defined with respect to the sequence of optimal actions as
\[
R_N=\sum_{n=1}^N (g_n(a_n^*) - g_n(A_n)) \,,
\]
where $A_n$ denotes the arm selected by the algorithm at time $n$.
Define 
\[
S(x_1,\ldots,x_N) = \sum_{n=1}^{N-1} \one\{x_{n}\neq x_{n+1}\} \;.
\]
We denote by $S = S(a_1^*,\ldots,a_N^*)$ the number of times the identity of the optimal arm changes. Our goal is to design an algorithm that satisfies the following regret bound without the knowledge of $S$:
\begin{equation}
\label{eq:tracking3}
\E[R_N] = O\left(\sqrt{K N S(a_1^*,\ldots,a_N^*)} \right)\;.
\end{equation}
We emphasize that this goal is not addressed by literature on stochastic non-stationary bandits. In most of the existing results, regret bounds depend on the number of changes in the reward distributions $D_n$ instead, which can be much larger than $S$.

\section{Algorithm}

\begin{algorithm}[ht!]
\caption{\algoname algorithm}
\label{alg:armswitch}
 \textbf{Input:} time horizon $N$, constant $\delta \in (0, 1)$
\begin{algorithmic}[1]
\STATE Initialize $s=0, n=0$.
\STATE \textbf{Start a new episode:} \label{ln:new_episode} \\
 $s \gets s+1$, $t_s \gets n+1$\\
 $\GOOD \gets \{1, \ldots K \}$, $\BAD \gets \emptyset$.\\
 $\text{Active}(a) \gets n+1,\, 
 \forall a$. 
\STATE \textbf{Next time step:}  $n \gets n+1$ \label{ln:new_step} \\
\COMMENT{** History $\cH'_n$ is the information available to the algorithm at this point **}
\FOR{$a \in \BAD$}\label{ln:explo1}
\FOR{$\varepsilon \in {\cal B} = \{ 2^{-1}, 2^{-2},  \ldots 2^{-\lceil\log_2 N \rceil} \}$} \label{ln:randomchoice}
\STATE With probability $\varepsilon \sqrt{s / (KN)}$:
\IF {$\B(a) \leq 0$}
\STATE Set $\text{Active}(a) \gets n$.\label{ln:active}
\ENDIF
\STATE $\B(a) \gets \max(\B(a), 1/\varepsilon^2)$.
\label{ln:explo2}
\ENDFOR
\ENDFOR \label{ln:explo3} 
\STATE Define the active set $\cA = \GOOD \cup \{a
\in \BAD: \B(a)\ge \frac{1}{K}\}$. \label{ln:activeset}
\STATE Set $\B(a) \gets 0$ for all $a \not\in \cA$, and let $m=|\cA \cap \BAD|$. \label{ln:S-update1} \\
\COMMENT{** History $\cH_n$ is the information available to the algorithm at this point **}
\STATE Define distribution $P_n$ by $P_n(a)=\frac{1}{K}$ for $a\in \cA \cap \BAD$, and $P_n(a)=\frac{1-m/K}{|\GOOD|}$ for $a\in \GOOD$
\STATE Select $A_n$ by sampling from $P_n$ and receive reward $r_n$.\\ \label{ln:sampling}
\COMMENT{** Define variables with index $n$: $\GOOD_n=\GOOD$, $\BAD_n=\BAD$, $\B_n(a)=\B(a)$, etc. **}
\label{ln:index}
\STATE Set $\B(a) \gets \B(a) - \frac{1}{K}$ for all $a\in \cA \cap \BAD$. \label{ln:S-update}
\FOR {$a, a' \in \cA$}
\label{ln:elim1}
\IF {$\better{a'}{a}{n}$}
\STATE If $a \in \GOOD$, move $a$ to $\BAD$.\label{ln:a-to-bad}
\STATE If $a \in \BAD$, set $\B(a) \gets 0$. \label{ln:explo-zero}
\ENDIF \label{ln:elim2}
\IF {$\GOOD$ is empty}
\STATE Go to \ref{ln:new_episode} (start a new episode)
\ENDIF
\ENDFOR
\STATE Go  to \ref{ln:new_step} (next time step).
\end{algorithmic}
\end{algorithm}

\begin{algorithm}[ht!]
\caption{The $\better{a'}{a}{n}$ subroutine}
\label{alg:elim}
\begin{algorithmic}[1]
\FOR {$n' \in [\max\{\text{Active}_n(a), \text{Active}_n(a')\}, n]$}
\IF {$\widehat{\widetilde \Delta}_{n':n}(a',a) > 12 C_{n',n} \left(\sqrt{\frac{n-n'+1}{K}} \vee C_{n',n} \right)$, or $a,a'\in \GOOD$ and $\widehat \Delta_{n':n}(a',a) > 12 C_{n',n} \left(\sqrt{P_{n':n}} \vee C_{n',n} \right)$ 
} 
\RETURN true
\ENDIF
\ENDFOR
\RETURN false 
\end{algorithmic}
\end{algorithm}

To develop our algorithm, we start with the \textsc{AdSwitch} method of \citet{AGO-2019},
which was developed for the standard piecewise stationary scenario, where the reward statistics may change only a limited number of times (say $M$), and they remain constant inbetween. Conceptually, \textsc{AdSwitch} works the following way. For every stationary segment it maintains a set of arms which can be optimal for the given segment (called the \GOOD set), and these arms are pulled in a round robin fashion. The expected reward of each arm is estimated based on the observations (calculating the average reward for each arm together with confidence intervals), and if the estimates imply with high probability that an arm cannot be optimal, it is removed from the \GOOD set. This phased elimination strategy is known to achieve an optimal regret rate for stationary stochastic bandits. To be able to detect the end of a stationary segment, \textsc{AdSwitch} also explores arms which are not in the GOOD set (the set of those arms is called the BAD set), with a carefully designed exploration strategy, striking a good balance between the exploration probability and the amount of change the selected exploration strategy can detect.
If a change is detected (with high probability), the algorithm declares the end of the current segment, and it is restarted, with all the arms being in the GOOD set and all estimates reset.

Our algorithm, called \algoname (for considering the number of times the optimal \emph{arm} changes), is based on similar ideas, but we need to introduce several changes in both the algorithm and its analysis to make it work in our setting.
Similarly to \textsc{AdSwitch}, \algoname tries to identify adaptively the segments where the identity of the optimal arm remains the same, and also to learn the optimal arm in each segment. As such, we also maintain a \GOOD set of arms, containing the arms which can be optimal in the given segment, also carefully explore the arms in the \BAD set.\footnote{Throughout the paper, when these sets are referred to in a specific time step, they are indexed with that time step (as mentioned in the comment in line~\ref{ln:index} of the algorithm), and we also refer to the arms in the \GOOD, resp.\ \BAD, set as \GOOD, resp.\ \BAD, arms.} Note, however, that because in our case the mean rewards can change arbitrarily for every arm, deterministically going over all active arms (the \GOOD arms and the ones being explored, together the active arms $\cA$) and taking the average observed reward for every arm will not give good estimates of the arms' performances over the segment. Instead, we randomly sample from the active set $\cA$, and when comparing the performance of two arms over an interval where both arms were active, we compare their cumulative weighted reward where each reward is weighted with the sampling probability (instead of the average observed reward) irrespective of how many times the arms were actually used. We provide more details on the sampling and comparison procedures in the next section. Most notably, the sampling distribution is non-uniform and assigns a fixed probability of $1/K$ to any \BAD arm in $\cA$ (i.e., to arms in $\BAD \cap \cA$).

\subsection{The \algoname algorithm}
\label{sec:algdef}

Our algorithm is shown in Algorithm~\ref{alg:armswitch}. 
It is an elimination algorithm with repeated exploration and restarts. The algorithm proceeds in episodes $s = 1, 2, ...$; we use $t_s$ to denote the start time of the $s^{th}$ episode. 

The algorithm continuously maintains a set of arms \GOOD containing the arms which can be optimal in the current episode, while $\BAD=[K] \setminus \GOOD$ contains all other arms. At the beginning of each episode, \GOOD contains all the arms, which are then continuously eliminated (and moved to \BAD) when it can be proved with high probability based on the received rewards that they cannot be optimal for at least one time step  
of the episode. When all arms are eliminated from the \GOOD set, the algorithm knows that the identity of the optimal arm has changed with high probability, so a new episode is started (where again all arms can be optimal initially). In principle, in every time step the algorithm plays the arms in the \GOOD set uniformly at random. However, this would miss detecting if an arm from the \BAD set became good. To handle such cases, in every step \algoname may select, with some small probability, some arms from the \BAD set to be explored, and the algorithm may play these arms as well, beside the ones in the \GOOD set; the arms which can be played in any given time step are collected in the active set $\cA$. 

Next we describe each component of \algoname in detail:

\paragraph{Exploration of arms in the \BAD set.} 
To facilitate exploration of arms in the \BAD set, \algoname
maintains so called exploration obligations $\B(a)$, containing a prescribed sum of probabilities with which a bad arm has to be explored over multiple time steps (this is similar to exploration obligations used by \textsc{AdSwitch} prescribing how many times an arm has to be sampled, but it is better tuned for random sampling). When any arm gets to the \BAD set, its exploration obligation is set to 0 (lines~\ref{ln:a-to-bad}
and ~\ref{ln:explo-zero}). At the beginning of every time step (before an arm is played, see lines~\ref{ln:explo1}--\ref{ln:explo3}), the algorithm may schedule some exploration for some arms in the \BAD set:
The obligation (length) can be $1/\varepsilon^2$ for any $\varepsilon$ in an exponential grid $\cB = \{1/2, 1/4, \ldots , 2^{-\lceil \log_2 N \rceil}\}$, and longer explorations have smaller probabilities: in episode $s$,  $a \in \BAD_n$ and $\varepsilon \in \cB$, with probability $\varepsilon \sqrt{\frac{s}{K N}}$, we prescribe an exploration obligation of length $1/\varepsilon^2$ by setting $\B(a)$ to $\max(\B(a), 1/\varepsilon^2)$. We say that an obligation is \emph{scheduled} in a time step for arm $a$ if it is the longest obligation prescribed, and it is larger than the previous obligation, and define a corresponding event  $\EXP(a,n,\varepsilon)$ which holds if and only if an exploration obligation of length $1/\varepsilon^2$ is scheduled for arm $a$ in time step $n$ (i.e., $\B_n(a)=1/\varepsilon^2>\B_{n-1}(a)$). Note that the conditional probability of $\EXP(a,n,\varepsilon)$ (given the history up to this point) if $n$ belongs to episode $s$ is at most $\varepsilon \sqrt{\frac{s}{K N}}$.

Arms in the \BAD set with positive exploration obligations typically belong to the active set $\cA$ for the given time step, 
unless their exploration obligation is ``rounded down'' to 0 (see line~\ref{ln:S-update1}, and the description of the procedure for sampling an arm below).
After the algorithm plays an arm and receives a reward, the exploration obligations for any active arm in the \BAD set are reduced by the (conditional) probability for selecting that arm, that is, by $\frac{1}{K}$ (line~\ref{ln:S-update}).

\paragraph{Selecting arms.}
After possibly introducing new sampling obligations, the algorithm selects a set of active arms $\cA$ from which the played arm $A_n$ is selected. Set $\cA$ contains all the arms in the \GOOD set, and also all the \BAD arms with exploration obligations at least  $\frac{1}{K}$. The algorithm selects an action from a distribution that assigns probability $\frac{1}{K}$ to any \BAD arm in the active set, and is uniform over the \GOOD arms using the remaining probabilities (line~\ref{ln:activeset}). It will be helpful to define an additional variable $\widetilde A_n$ by implementing the sampling of $A_n$ in a two-step procedure: with probability $\frac{|\cA|}{K}$, an arm is sampled uniformly at random from $\cA$, and otherwise an arm is sampled uniformly at random from $\GOOD$. The sampled arm is denoted by $A_n$; if the first event happens, we also let $\widetilde A_n=A_n$, while in case of the second event, we let $\widetilde A_n=*$ to take a value not in the action set (an alternative definition of $\widetilde A_n$ is to choose $A_n$ according to $P_n$ and then set  $\widetilde A_n$ to  $A_n$ with probability $\tfrac{1}{K}/P_n(A_n)$ and to $*$ otherwise).

\paragraph{Eliminating arms from the active set.}
If, based on the observed rewards $r_n$ and the selected actions, the algorithm can prove (with high probability) that an arm in the active set $\cA$ cannot be optimal starting from the last time it has become active, it is removed from the active set, that is, it is removed from the set of \GOOD arms if it belonged there,  or its exploration obligations are deleted if it belongs to the \BAD set (lines~\ref{ln:elim1}--\ref{ln:elim2}). 
This elimination is based on observations in intervals when an arm is \emph{active}: we say that arm $a$ is active in time step $n$, if $a \in \cA_n$.
The start of the most recent active period is maintained in the variable Active: for arms in the \GOOD set, it is the beginning of the current episode (line~\ref{ln:new_episode}), while for active arms $a$ in the \BAD set, it is the time when the algorithm has decided to explore them (i.e., when the exploration obligation $\B(a)$ last became positive, see line~\ref{ln:active}). If $a \in \cA_n$, then $a$ is active in $[n':n]$ if and only if $\text{Active}_n(a) \le n'$.
We denote by the Boolean $\better{a'}{a}{n}$ if we can prove (with high probability) in time step $n$ that arm $a'$ was better in at least one time step than $a$ during a time interval ending at time $n$ when both arms were active. In this case we know that $a$ cannot be an optimal arm in this interval, hence it cannot be an optimal arm in the episode, and so it can be eliminated. The details of this procedure, which is at the heart of \algoname, are given in the next section (Section~\ref{sec:comparison}) and Algorithm~\ref{alg:elim}.

\subsection{Comparison of arms ($\better{a'}{a}{n}$)}
\label{sec:comparison}

We now specify the arm-elimination condition $\better{a'}{a}{n}$, which indicates whether arm $a'$ is better than $a$ in at least one time step in the current episode up to time $n$, with high probability, given the algorithm's history up to that point. If it is the case, arm $a$ can be eliminated from the active set by $a'$. The formal definition is given in Algorithm~\ref{alg:elim}.

Let $P_n(a)$ denote the probability of selecting arm $a$ in time step $n$, given the history $\cH_n$ up to that point (see line~\ref{ln:S-update1} of Algorithm~\ref{alg:armswitch}). By definition, $P_n(a)=\frac{1}{K}$ for $a\in \cA_n \cap \BAD_n$, and $P_n(a)=\frac{1}{|\GOOD_n|}\left(1 - \frac{|\cA_n \cap \BAD_n|}{K} \right)$ for $a\in \GOOD_n$ (note that the latter is always at least $1/K$). For $a\notin \cA_n$, $P_n(a)=0$. 
Let $A_n$ be the arm selected at time $n$, $G_n(a):=P_n(a) g_n(a)$ be the weighted expected reward of arm $a$ in time step $n$, and $\widehat G_n(a) := \one\{A_n=a\} r_n$ be its estimate; note that this is an unbiased estimate since $G_n(a) = \E[\widehat G_n(a)|\cH_n]$. We also define $\widetilde G_n(a):=\frac{g_n(a)}{K}$ and its estimate $\widehat{\widetilde G}_n(a) := \one\{\widetilde A_n=a\} r_n$ (recall that $\widetilde A_n$ equals $A_n$ with probability $1/(K P_n(A_n))$ and $*$ otherwise, see its definition in the paragraph on selecting arms in Section~\ref{sec:algdef}). For any interval $[n':n]$, we consider the following weighted sums of (possibly expected) rewards:
\begin{align*}
    G_{n':n}(a) &:= \sum_{t\in [n':n]} P_t(a) g_t(a)\,,\qquad &\widehat G_{n':n}(a)&:= \sum_{t\in [n':n]} \widehat G_t(a) = \sum_{t\in [n':n]} \one\{A_t=a\} r_t\,, \\
    \widetilde G_{n':n}(a) &:= \frac{1}{K}\sum_{t\in [n':n]} g_t(a)\,,\qquad &\widehat{\widetilde G}_{n':n}(a)&:= \sum_{t\in [n':n]} \widehat{\widetilde G}_t(a) = \sum_{t\in [n':n]} \one\{\widetilde A_t=a\} r_t \,. 
\end{align*}
Note that $\widetilde G_{n':n}(a)=G_{n':n}(a)$ and $\widehat{\widetilde G}_{n':n}(a)=\widehat G_{n':n}(a)$ for $a\in \bigcap_{t=n'}^n \cA_t\cap\BAD_t$.%
\footnote{Also notice that since $\one\{A_t=a\}$ is a Bernoulli random variable whose variance is $P_t(a)$, $G_{n':n}(a)$ is similar in spirit to the minimum-variance estimator of a joint mean of independent random variables, where the optimal weighting is proportional to the variance of each variable.}
Define $\Delta_{n':n}(a,a')$, $\widehat \Delta_{n':n}(a,a')$, $\widetilde\Delta_{n':n}(a,a')$ and $\widehat{\widetilde \Delta}_{n':n}(a,a')$ as follows:
for any arms $a,a' \in [K]$, let
\begin{align*}
\Delta_{n':n}(a,a') &:= G_{n':n}(a)-G_{n':n}(a')\,,\qquad &\widehat \Delta_{n':n}(a,a') &:= \widehat G_{n':n}(a)- \widehat G_{n':n}(a') \,, \\
\widetilde\Delta_{n':n}(a,a') &:= \widetilde G_{n':n}(a)- \widetilde G_{n':n}(a')\,,\qquad &\widehat{\widetilde \Delta}_{n':n}(a,a') &:= \widehat{\widetilde G}_{n':n}(a)- \widehat{\widetilde G}_{n':n}(a') \;.
\end{align*}
Note that $\widetilde\Delta_{n',n}(a,a')=-\widetilde\Delta_{n',n}(a',a)$ and $\widehat{\widetilde\Delta}_{n',n}(a,a')=-\widehat {\widetilde\Delta}_{n',n}(a',a)$.
Since $\widehat{\widetilde G}_{n}(a)$ is an unbiased estimate of $\widetilde G_{n}(a)$ as explained above, using martingale concentration it can be shown that $\widehat{\widetilde G}_{n',n}(a)$ is close to $\widetilde G_{n',n}(a)$ (and in turn $\widehat {\widetilde\Delta}_{n',n}(a,a')$ is close to $\widetilde\Delta_{n',n}(a,a')$): defining $C_{n',n}=
\sqrt{\log\left(\frac{2K N^2(\log(n-n'+1)+2)}{\delta}\right)}$
for some $\delta \in (0,1)$, Lemma~\ref{lemma:hoeff} shows that
\begin{align}
\label{eq:ghat}
\left|\widehat G_{n':n}(a) - G_{n':n}(a) \right| &\le 6\, C_{n',n} \left(\sqrt{P_{n':n}(a)} \vee C_{n',n} \right)\,,\\
\label{eq:gphat}
\left|\widehat{\widetilde G}_{n':n}(a) - \widetilde G_{n':n}(a) \right| &\le 6\, C_{n',n} \left(\sqrt{\frac{n-n'+1}{K}} \vee C_{n',n} \right)
\end{align}
with probability at least $1-\delta$ simultaneously for all actions $a \in [K]$ and interval $[n':n]\subset [N]$.

If $a$ is an optimal arm in the interval $[n':n]$, then $g_t(a) \ge g_t(a')$ for any other arm $a'$ and $t \in [n':n]$. Therefore, if both $a$ and $a'$ are active in $[n':n]$, $\widetilde\Delta_{n',n}(a,a') \ge 0$. Thus, if for an arm $a$ there exists another arm $a'$ such that $\widetilde\Delta_{n',n}(a,a')<0$ --- or equivalently, $\widetilde\Delta_{n',n}(a',a)>0$ --- then $a$ cannot be optimal, and hence can be eliminated from the set of potentially optimal arms for $[n':n]$. We use our empirical estimates $\widehat{\widetilde \Delta}_{n':n}(a,a')$ to verify, with high probability, if this happens: if \eqref{eq:gphat} holds for $a,a'$, then
\begin{align}
\label{eq:elimp-condition}
\widehat{\widetilde \Delta}_{n':n}(a',a) > 12 C_{n',n} \left(\sqrt{\frac{n-n'+1}{K}} \vee C_{n',n} \right)
\end{align}
implies that $\widetilde \Delta_{n':n}(a',a)>0$. The indicator $\better{a'}{a}{n}$ is true (see Algorithm~\ref{alg:elim}) if \eqref{eq:elimp-condition} holds for any interval $[n':n] \subset [t_s:n]$ such that both $a$ and $a'$ are active on $[n':n]$. 

With a similar argument, we can show that if \eqref{eq:ghat} holds for $a,a'\in\GOOD_n$, then
\begin{align}
\label{eq:elim-condition}
\widehat{\Delta}_{n':n}(a',a) > 12 C_{n',n} \left(\sqrt{P_{n':n}(a)} \vee C_{n',n} \right)
\end{align}
implies that $\Delta_{n':n}(a',a)>0$. The indicator $\better{a'}{a}{n}$ is true if \eqref{eq:elim-condition} holds for any interval $[n':n] \subset [t_s:n]$ such that both $a$ and $a'$ are in the \GOOD set.

\subsection{Regret of \algoname}

The following theorem shows that the regret of \algoname behaves as desired.
\begin{theorem}
\label{thm:regret}
For a switching multi-armed bandit problem with $K \ge 2$ arms, horizon $N\ge 2$, and $S \ge 0$ changes in the identity of the optimal arm, the expected regret of \algoname run with $\delta=1/N$ can be bounded as
$\E[R_N] \leq \CONST\sqrt{K(S+1)N}\log^2(KN)$ for an appropriate universal constant $\CONST$.
\end{theorem}

The rest of the paper (Section~\ref{sec:proof}) is devoted to the proof of this theorem. In what follows, we will use $\CONST$ to denote a constant whose value is independent of $K,S,N$, but might change from line to line. 

\section{Analysis}
\label{sec:proof}

We start with a few useful lemmas then analyze the regret in Section~\ref{sec:regretproof}.

\subsection{Useful lemmas}

Fix $\delta \in (0,1)$. We define events $\cE_1$, $\cE_2$ under which the estimates
$\widehat G_{n',n}(a)$, $\widehat{\widetilde G}_{n',n}(a)$ are good:
\begin{align*}
\cE_1 & = \Bigg\{
\left|\widehat G_{n':n}(a) - G_{n':n}(a) \right| \le 6 C_{n',n} \left(\sqrt{P_{n':n}(a)} \vee C_{n',n} \right) \text{ for all $[n':n] \subseteq [N]$ and actions $a \in [K]$}
\Bigg\}\,,\\
\cE_2 & = \Bigg\{\left|\widehat{\widetilde G}_{n':n}(a) - \widetilde G_{n':n}(a) \right| \le 6 C_{n',n} \left(\sqrt{\frac{n-n'+1}{K}} \vee C_{n',n} \right) \text{ for all $[n':n] \subseteq [N]$ and actions $a \in [K]$}
\Bigg\}\,.
\end{align*}
To help the analysis, it will also be useful to consider the following quantities for any $n' \le n$ and arms $a,a' \in [K]$:
\begin{align*}
G_{n':n}(a,a'):= \sum_{t=n'}^n P_t(a) g_t(a')\;, \qquad G'_{n':n}(a,a') := \sum_{t=n'}^n \one\{A_t=a\}g_t(a')\;,
\end{align*}
\begin{align*}
\widetilde{\Delta'}_{n':n}(a,a') := \sum_{t=n'}^n \one\{\widetilde A_t=a\}\big(g_t(a)-g_t(a')\big)\;.
\end{align*}
Note that by taking conditional expectations it follows that  $G'_{n':n}(a,a')$ should be close to $G_{n':n}(a,a')$ and $\widetilde{\Delta'}_{n':n}(a,a')$ should be close to $\widetilde{\Delta}_{n':n}(a,a')$. 
Defining $C'_{n',n}=\sqrt{\log\left(\frac{2K^2 N^2(\log(n-n'+1)+2)}{\delta}\right)}$ (which satisfies $C'_{n',n}\le 2 C_{n',n}$), these are formalized in the following events:
\begin{align*}
\cE_3 & = \Bigg\{\left|G'_{n':n}(a,a') - G_{n':n}(a,a') \right| \le 5 C'_{n',n} \left(\sqrt{P_{n':n}(a)} \vee C'_{n',n} \right) \text{ for all $[n':n] \subseteq [N]$ and actions $a,a' \in [K]$}
\Bigg\}\,,\\
\cE_4 & = \Bigg\{\left|\widetilde{\Delta'}_{n':n}(a,a') - \widetilde \Delta_{n':n}(a,a') \right| \le 5 C'_{n',n} \left(\sqrt{\frac{n-n'+1}{K}} \vee C'_{n',n} \right) \text{ for all $[n':n] \subseteq [N]$ and actions $a,a' \in [K]$}
\Bigg\}\,.
\end{align*}

The next lemma shows that $\cE_1$, $\cE_2$, $\cE_3$, and $\cE_4$ hold with high probability. Its proof, presented in Appendix~\ref{sec:proof_lemmas}, is based on a version of Freedman's inequality.
\begin{lemma}\label{lemma:hoeff}
Each of the events $\cE_1$, $\cE_2$, $\cE_3$, and $\cE_4$ hold with probability at least $1-\delta$.
\end{lemma}

We bound the regret of \algoname under the event $\cE_1\cap\cE_2\cap\cE_3\cap\cE_4$. 
Taking into account that the regret in every time step is at most 1 and choosing $\delta=1/N$, the contribution to the expected regret when $\cE_1$, $\cE_2$, $\cE_3$ or $\cE_4$ do not hold can be bounded by 4. 
For arms $a,a'$, note that since
\begin{align*}
    \widehat {\Delta}_{n':n}(a',a) &= \Big(\widehat G_{n':n}(a') - G_{n':n}(a') \Big)
    +  \Big(G_{n':n}(a') - G_{n':n}(a) \Big)
    + \Big(G_{n':n}(a) - \widehat G_{n':n}(a)\Big)\,,
\end{align*}
$\cE_1$ implies
\begin{align*}
    \left|\widehat {\Delta}_{n':n}(a',a)   -  
\Delta_{n':n}(a',a)\right| &\le 6 C_{n',n} \left(\sqrt{P_{n':n}(a)} \vee C_{n',n} + \sqrt{P_{n':n}(a')} \vee C_{n',n}\right)\;.  
\end{align*}
In particular, if both $a$ and $a'$ are \GOOD throughout the interval $[n':n]$, then $P_t(a)=P_t(a')$ for all $t \in [n':n]$, and hence 
\begin{align}
    \left|\widehat {\Delta}_{n':n}(a',a)   -  
\Delta_{n':n}(a',a)\right| &\le 12 C_{n',n} \left(\sqrt{P_{n':n}(a)} \vee C_{n',n}\right)\;.  \label{eq:eqpDelta}
\end{align}
Similarly, $\cE_2$ implies, for any $a,a' \in [K]$,
\begin{align}
    \left|\widehat {\widetilde\Delta}_{n':n}(a',a)   -  
\widetilde\Delta_{n':n}(a',a)\right| &\le 12 C_{n',n} \left(\sqrt{\frac{n-n'+1}{K}} \vee C_{n',n} \right)\;.  \label{eq:eqDelta}
\end{align}

\begin{lemma}
\label{lem:activeArms}
Let $n' \le n$ be two time steps belonging to the same episode.
\begin{enumerate}[(i)]
\item Suppose $\cE_1$ holds. If $\Delta_{n':n}(a',a) > 24 C_{n',n} \left(\sqrt{P_{n':n}} \vee C_{n',n} \right)$ for $a,a' \in\GOOD_n$, then $\better{a'}{a}{n}$
is true. Furthermore, if $a,a'\in\GOOD_{n+1}$, then $\better{a'}{a}{n}$ is false and
$\Delta_{n':n}(a',a) \le 24 C_{n',n} \left(\sqrt{P_{n':n}(a)} \vee C_{n',n} \right)$. 
\item Assume $\cE_2$ holds and arms $a, a'$ are active on $[n':n]$. If $\widetilde \Delta_{n':n}(a',a) > 24 C_{n',n} \left(\sqrt{\frac{n-n'+1}{K}} \vee C_{n',n} \right)$, then $\better{a'}{a}{n}$
is true. Furthermore, if $a'\in\GOOD_{n+1}$ and either $a\in \GOOD_{n+1}$ or $a \in \BAD_{n+1} \cap \cA_{n+1}$ with no new exploration obligation scheduled for $a$ in time step $n+1$, then $\better{a'}{a}{n}$ is false and
$\widetilde \Delta_{n':n}(a',a) \le 24 C_{n',n} \left(\sqrt{\frac{n-n'+1}{K}} \vee C_{n',n} \right)$.
\end{enumerate}
\end{lemma}

\subsection{Preliminaries to the proof of Theorem~\ref{thm:regret}}
\label{sec:regretproof}

We bound the regret of \algoname under the event $\cE_1\cap\cE_2\cap\cE_3\cap\cE_4$.
We first show that the number of episodes produced by our algorithm is at most $S+1$.

\begin{lemma}
\label{lem:episodes}
If $\cE_1\cap\cE_2$ holds, \algoname has at most $S+1$ episodes.
\end{lemma}
\begin{proof}
At the beginning of each episode $s$, the set $\GOOD_{t_s}$ includes all arms including the current optimal arm. Under the event $\cE_1\cap\cE_2$, the arm that is optimal at the start of the episode cannot be eliminated from the GOOD set for as long as it stays optimal. Thus, the GOOD set can only become empty after a change in the optimal arm, which happens $S$ times.
\end{proof}

In what follows, let $a_s^g$ denote an arm that stays in the \GOOD set during the entire episode $s$ (i.e. the last arm to get eliminated, or one of the last arms if multiple arms are eliminated when \GOOD becomes empty).
Let $\cN_s$ denote the set of time steps corresponding to episode $s$. We decompose the regret in each episode $s$ into the regret of played arms with respect to $a_s^g$, and the regret of $a_s^g$ with respect to the optimal arms as follows:
\begin{align*}
R_N &= \sum_{n=1}^N (g_n(a_n^*) - g_n(A_n)) 
 = R_N^{(1)} + R_N^{(2)} + R_N^{(3)}\,,
\end{align*}
where we define
\begin{align*}
R_N^{(1)} &:= \sum_{s=1}^S \sum_{n\in \cN_s} \one\{A_n\in \BAD_n\}(g_n(a_s^g) - g_n(A_n)) \,, \\
R_N^{(2)} &:= \sum_{s=1}^S \sum_{n\in \cN_s} \one\{A_n\in \GOOD_n\}(g_n(a_s^g) - g_n(A_n)) \,, \\
R_N^{(3)} &:= \sum_{s=1}^S \sum_{n\in \cN_s} (g_n(a_n^*) - g_n(a_s^g)) \;.
\end{align*}

We proceed by bounding each of the above regret terms. 
Bounding  $R_N^{(1)}$ and $R_N^{(2)}$ can be done using arguments similar to existing techniques of~\citet{AGO-2019}; this is presented in Section~\ref{sec:R1R2}, in Lemma~\ref{lem:R1} and Lemma~\ref{lem:R2}, respectively.

The main challenge is bounding $R_N^{(3)}$, for the following reason. Choosing the sampling distribution $P_n$ to be uniform over the active set, which is similar to the round robin action selection of~\citet{AGO-2019}, would lead to a suboptimal bound in $K$: when the optimal arm is fixed, in the \BAD set, and active on interval $[n':n]$, the condition for eliminating $a_s^g$ (and starting a new episode) is of the form
\[
\sum_{t=n'}^{n} \frac{1}{|\cA_t|} (g_t(a_t^*) - g_t(a_s^g)) \approx \sum_{t=n'}^{n} (\widehat G_t(a_t^*) - \widehat G_t(a_s^g)) \gtrapprox 12 C_N \sqrt{\sum_{t=n'}^{n} \frac{1}{|\cA_t|}}  \;,
\]
where $C_N:=C_{1,N}$ which equals $\sqrt{\log\big(2KN^3(\log(N)+2)\big)}$ for $\delta=1/N$. 
The problem here is that the right-hand side can be as large as $12 C_N \sqrt{n-n'+1}$ or as small as $12 C_N \sqrt{\frac{n-n'+1}{K}}$, depending on the size of the active sets $\cA_t$. The active sets, however, evolve in a complex manner, which makes it difficult to have a tight bound on $\sum_{t=n'}^{n} \frac{1}{|\cA_t|}$. In order to obtain tight bounds, we consider a non-uniform sampling distribution that assigns a fixed probability of $\frac{1}{K}$ to any active \BAD arm (including the optimal arm in the above scenario). Then the elimination condition takes the form of
\[
\sum_{t=n'}^{n} (g_t(a_t^*) - g_t(a_s^g)) \gtrapprox 12 C_N \sqrt{K(n-n'+1)}  \;.
\]
The interval $[n':n]$ is explored by the optimal arm with probability roughly $\sqrt{\frac{s}{N (n-n'+1)}}$, and this event triggers a restart; otherwise the learner suffers an $\widetilde O(C_N\sqrt{K (n-n'+1)})$ regret for this interval, which all together, as we will see in Section~\ref{sec:R3}, lead to an optimal (up to logarithmic factors) $\widetilde O(\sqrt{KSN})$ regret bound (note that $C_N = \widetilde O(1)$).
This is proved formally in Lemma~\ref{lem:R3}. Combining Lemmas~\ref{lem:R1}--\ref{lem:R3} trivially yields Theorem~\ref{thm:regret};
the rest of the paper is devoted to prove these lemmas.

\subsection{Bounding $R_N^{(1)}$ and $R_N^{(2)}$}
\label{sec:R1R2}

We start with bounding $\E[R_N^{(1)}]$.
\begin{lemma}
\label{lem:R1}
Under the conditions of Theorem~\ref{thm:regret}, if $N\ge 8$, we have
\[
\E[R_N^{(1)}] \le 34 C_N^2 \log_2(N) \left(\sqrt{K(S+1)N} + 2\sqrt{K}\right)+2~.
\]
\end{lemma}
\begin{proof}
Here we bound the regret due to exploring arms in the BAD set in each episode with respect to the arm $a_s^g$. Note that all sampling obligations start and end in the same episode, since all arms are placed in the GOOD set at the beginning of each episode, 

Consider a sampling obligation for arm $a\in \BAD_n$ that is scheduled in time step $n$ in episode $s$ with exploration parameter $\varepsilon$. Let $\tau_n(\varepsilon,a)$ be the time at which this particular sampling obligation expires, that is, either $n'=\tau_n(\varepsilon,a)+1$ is the first time after $n$ when the obligation $\B_{n'}(a)$ is zero (i.e., it is zero after line~\ref{ln:S-update1} of Algorithm~\ref{alg:armswitch}), or $n'=\tau_n(\varepsilon,a)+1$ is the first time step when a new exploration obligation is scheduled for $a$ (by triggering the $\max$ operation in line~\ref{ln:explo2}).
Then, under $\cE_2 \cap \cE_4$, the regret with respect to $a_s^g$ can be bounded as
\begin{align}
\label{eq:expregret1}
\notag
\lefteqn{
\sum_{t\in [n:\tau_n(\varepsilon,a)]} \one\{A_t=a\} (g_t (a_s^g) - g_t(a))
} \\
&\le 
\widetilde\Delta_{n:\tau_n(\varepsilon,a)}(a_s^g,a) + 10C_N\left(\sqrt{\frac{\tau_n(\varepsilon,a)-n+1}{K}} \vee C_N \right) \\
\notag
&\le 34 C_N\left(\sqrt{\frac{\tau_n(\varepsilon,a)-n+1}{K}} \vee C_N \right) + 1 \\
&\le 34C_N \left(\frac{1}{\varepsilon} \vee C_N \right) + 1 \,,
\end{align}
where the first inequality holds by the definition of $\cE_4$ because $a$ is a \BAD arm in $[n:\tau_n(\varepsilon,a)]$ (hence $A_t=a$ is equivalent to $\widetilde A_t=a$) and $C'_{n',n} \le 2 C_N$; the second inequality holds by Lemma~\ref{lem:activeArms}(ii) (since we assumed that $\cE_2$ holds) and the fact that the reward difference between any two arms in one step is at most 1; and the last inequality holds by the fact that the weighted length of the interval cannot be more than $1/\varepsilon^2$ (i.e., $\tau_n(\varepsilon,a)-n+1 \le K/\varepsilon^2$). 

Recall that $\cN_s$ denotes the set of time steps corresponding to episode $s$ and that $\EXP(a,n,\varepsilon)$ denotes the event that an exploration obligation of length $1/\varepsilon^2$ is scheduled for arm $a$ in time step $n$.
Note that the exploration intervals $[n,\tau_n(\varepsilon_n,a)] \subset \cN_s$ for the scheduled explorations of arm $a$ (that is, for which $\EXP(a,n,\varepsilon_n)$ hold) are disjoint for all $n$, and together cover all time steps in episode $s$ where $a$ belongs to the \BAD set and is active.

Using these observations, the expected regret due to exploring BAD arms in episode $s$ with respect to arm $a_s^g$ can be bounded as follows:
\begin{align*}
\E[R_{N,s}^{(1)}] &:= \E\left[\sum_{n\in \cN_s} \one\{A_n\in \BAD_n\} (g_n (a_s^g) - g_n(A_n)) \right]\\
&= \E\left[\sum_{n\in \cN_s} \sum_{a\in \BAD_n} \one\{A_n=a\} (g_n (a_s^g) - g_n(a)) \right]\\
&= \E\left[\sum_{n\in \cN_s} \sum_{a\in \BAD_n} \sum_{\varepsilon\in \cB} \one\{\EXP(a,n,\varepsilon)\} \sum_{t\in [n:\tau_n(\varepsilon,a)]} \one\{A_t=a\} (g_t (a_s^g) - g_t(a)) \right]\\
&\le \E\left[\sum_{n\in \cN_s} \sum_{a\in \BAD_n} \sum_{\varepsilon\in \cB} \one\{\EXP(a,n,\varepsilon)\} \one\{\cE_2 \cap \cE_4\}\sum_{t\in [n:\tau_n(\varepsilon,a)]} \one\{A_t=a\} (g_t (a_s^g) - g_t(a)) \right] +\E\left[\one\{\overline\cE_2 \cup \overline\cE_4\}|\cN_s|\right]\\
& \le \E\left[\sum_{n\in \cN_s} \sum_{a\in \BAD_n} \sum_{\varepsilon\in \cB} \one\{\EXP(a,n,\varepsilon)\}
\left(
34 C_N \left(\frac{1}{\varepsilon} \vee C_N \right) + 1
\right)
\right] 
+\E\left[\one\{\overline\cE_2 \cup \overline\cE_4\}|\cN_s|\right]
\end{align*}
where the first inequality holds trivially when introducing the indicator for $\cE_2 \cap \cE_4$ as the sum is always at most $|\cN_s|$, and the second inequality holds by \eqref{eq:expregret1}.
Taking into account that the conditional probability of $\EXP(a,n,\varepsilon)$ for $n \in \cN_s$, given all the information up to the beginning of time step $n$, which also determines if $n \in \cN_s$, is at most $\varepsilon\sqrt{s/(KN)}$, the first term above can be bounded as
\begin{align*}
\lefteqn{\E\left[\sum_{n\in \cN_s} \sum_{a\in \BAD_n} \sum_{\varepsilon\in \cB} \one\{\EXP(a,n,\varepsilon)\}
\left(
34 C_N \left(\frac{1}{\varepsilon} \vee C_N \right) + 1
\right)
\right] } \\
& \le 
\E\left[\sum_{n\in \cN_s} \sum_{a\in \BAD_n} \sum_{\varepsilon\in \cB} \sqrt{\frac{s}{KN}}
\big(
34C_N \left(1 \vee \varepsilon C_N \right) + \varepsilon
\big)
\right] \\
& \le \E\left[\sum_{n\in \cN_s} \sum_{a\in \BAD_n}
34C_N \sqrt{\frac{s}{KN}} \big(
\log_2(N)+2 +C_N\big)
\right] \\
& \le 34 C_N^2 \log_2(N) \sqrt{\frac{s}{KN}}
\E\left[\sum_{n \in \cN_s} |\BAD_n|\right] \\
&\le 34 C_N^2 \log_2(N) \sqrt{\frac{sK}{N}} \E[|\cN_s|]
\end{align*}
where in the second inequality we used that $|\cB| \le \log_2(N)+1$ and that $\sum_{\varepsilon \in \cB} \varepsilon < 1$, and in the last one that $N \ge 8, K \ge 2$ and that for $\delta=1/N$, $C_N=\sqrt{\log\big(2KN^3(\log(N)+2)\big)}$.

Summing up the above bound we obtain
\begin{align*}
\E[R_N^{(1)}] &= \E\left[\sum_{s} R_{N,s}^{(1)}\right] \\
& \le 34 C_N^2 \log_2(N) \sqrt{\frac{K}{N}} \cdot \E\left[\sum_s \sqrt{s}|\cN_s|\right]
+ \E\left[\one\{\overline\cE_2 \cup \overline\cE_4\}\sum_s |\cN_s|\right] \\
&= 34 C_N^2 \log_2(N) \sqrt{\frac{K}{N}} \left( \underbrace{\E\left[\one\{\cE_1\cap \cE_2\}\sum_s \sqrt{s}|\cN_s|\right]}_A
+ \underbrace{\E\left[\one\{\overline\cE_1 \cup \overline\cE_2\} \sum_s \sqrt{s}|\cN_s|\right]}_B
\right)
+ \underbrace{\E\left[\one\{\overline\cE_2 \cup \overline\cE_4\}\sum_s |\cN_s|\right]}_C 
\end{align*}
Since the number of episodes is at most $S+1$ under $\cE_1 \cap \cE_2$ by Lemma~\ref{lem:episodes}, and since the $\cN_s$ form a partition of $[1,N]$, we have $A \le N \sqrt{S+1}$.
To bound $B$, notice that $B \le M \E\left[\one\{\overline\cE_1 \cup \overline\cE_2\}\right]$
where
\begin{align*}
M &= \max_{k,n_1,\ldots,n_k \ge 1: \sum_{s=1}^k=N} \sum_{s=1}^k \sqrt{s} n_k 
\le \max_{k \in [N]} \sum_{s=1}^{k-1} \sqrt{s} + \sqrt{k}(N-k+1) \\
&\le \max_{k \in [N]} \frac{2}{3} (k-1)^{3/2} + \sqrt{k}(N-k+1) 
\le \max_{k \in [N]} k^{1/2} N \le N^{3/2}~.
\end{align*}
By Lemma~\ref{lemma:hoeff}, $\E\left[\one\{\overline\cE_1 \cup \overline\cE_2\}\right] \le 2/N$, and so $B \le 2 \sqrt{N}$.
Finally, by the same lemma, $\E\left[\one\{\overline\cE_2 \cup \overline\cE_4\}\right] \le 2/N$, and hence $C\le 2$.
Putting everything together, we obtain
\begin{align*}
\E[R_N^{(1)}]  
& \le 34 C_N^2 \log_2(N) \left(\sqrt{K(S+1)N} + 2\sqrt{K})\right)+2,
\end{align*}
as desired.
\end{proof}

\begin{lemma}
\label{lem:R2}
Under the conditions of Theorem~\ref{thm:regret}, we have
\begin{align}
\E\left[R_{N}^{(2)}\right]
\le 44 C_N\sqrt{(C_N^2 +1) K (S+1) N} +3\;. \label{eq:Goodag1}
\end{align}
\end{lemma}
\begin{proof}
For any $a \in [K]$, let $\cN_s(a)$ denote the set of time steps in episode $s$ on which arm $a$ is \GOOD; note that by definition, $\cN_s(a)$ is an interval of the form $[t_s,\tilde{t}_s(a)]$, where $\tilde{t}_s(a)$ denotes the last time step of episode $s$ in which $a$ is in the \GOOD set. Note that since $a_s^g$ in \GOOD throughout the whole episode, $P_n(a)=P_n(a_s^g)$, and we have
\begin{align*}
\lefteqn{\one\{A_n=a\}(g_n(a_s^g) - g_n(a))} \\
& =  g_n(a_s^g)\big(\one\{A_n=a\} - P_n(a)\big) + 
\big(P_n(a_s^g) g_n(a_s^g) -  \one\{A_n=a_s^g\} r_t \big)  \\
& \quad +
\big(\one\{A_n=a_s^g\} r_t - \one\{A_n=a\} r_t \big)
+ \big(\one\{A_n=a\} r_t - P_n(a) g_n(a)\big) + g_n(a)\big(P_n(a)  -\one\{A_n=a\} )\;.
\end{align*}
Summing up for all $n \in \cN_s(a)$, under $\cE_1 \cap \cE_3$ we get
\begin{align}
\lefteqn{\sum_{n \in \cN_s(a)} \one\{A_n=a\}(g_n(a_s^g) - g_n(a))} \nonumber \\
& \le \left(G'_{t_s:\tilde{t}_s(a)}(a,a_s^g) - G_{t_s:\tilde{t}_s(a)}(a,a_s^g)\right)
+ \left(G_{t_s:\tilde{t}_s(a)}(a_s^g)- \widehat G_{t_s:\tilde{t}_s(a)}(a_s^g)\right) \nonumber\\
& \quad + \widehat \Delta_{t_s:\tilde{t}_s(a)}(a_s^g,a)
+ \left(\widehat G_{t_s:\tilde{t}_s(a)}(a) - 
G_{t_s:\tilde{t}_s(a)}(a)\right)
+ \left(G_{t_s:\tilde{t}_s(a)}(a,a) - G'_{t_s:\tilde{t}_s(a)}(a,a)\right) \nonumber\\
& \le \widehat \Delta_{t_s:\tilde{t}_s(a)}(a_s^g,a)
+ 32 C_N \left(\sqrt{P_{t_s:\tilde{t}_s(a)}(a)} \vee C_N \right)
\label{eq:R2-split}
\end{align}

From here the regret of playing arms in the \GOOD set with respect to $a_s^g$ in episode $s$ can be bounded as follows:
\begin{align*}
\one\{\cE_1 \cap \cE_3\} R_{N,s}^{(2)} &:= \one\{\cE_1 \cap \cE_3\} \sum_{n\in \cN_s} \one\{A_n\in \GOOD_n\}\big(g_n(a_s^g) - g_n(A_n)\big) \\
&= \sum_{a\in [K]} \one\{\cE_1 \cap \cE_3\} \sum_{n\in \cN_s} \one\{A_n=a, a\in \GOOD_n\}\big(g_n(a_s^g) - g_n(a\big) \\
&= \sum_{a\in [K]} \one\{\cE_1 \cap \cE_3\} \sum_{n\in \cN_s(a)} \one\{A_n=a\} \big(g_n(a_s^g) - g_n(a)\big) \\
&\le \one\{\cE_1 \cap \cE_3\} \cdot \sum_{a\in [K]} \left(\widehat \Delta_{t_s:\tilde{t}_s(a)}(a_s^g, a) + 32 C_N \left(\sqrt{P_{t_s:\tilde{t}_s(a)}(a)} \vee C_N \right)\right)\\
&\le \one\{\cE_1 \cap \cE_3\} \cdot \sum_{a\in [K]} \left(44 C_N \left(\sqrt{P_{t_s:\tilde{t}_s(a)}(a)} \vee C_N \right) + 1\right) \\ 
&\le \one\{\cE_1 \cap \cE_3\} \cdot 44 C_N  \sum_{a\in [K]} \left(\sqrt{C_N^2 + P_{t_s:\tilde{t}_s(a)}(a)}  +1\right)  \\
&\le \one\{\cE_1 \cap \cE_3\} \cdot 44 C_N  \sum_{a\in [K]} \left(\sqrt{C_N^2 + P_{t_s:t_{s+1}-1}(a)}  +1\right) \\
&\le \one\{\cE_1 \cap \cE_3\} \cdot 44 C_N\sqrt{C_N^2 K^2 + K   \sum_{a\in [K]} \sum_{t\in \cN_s} P_t(a)} \\
&= \one\{\cE_1 \cap \cE_3\} \cdot 44 C_N\sqrt{C_N^2 K^2 + K |\cN_s|}\;,
\end{align*}
where the first inequality holds by \eqref{eq:R2-split}, the second because the elimination condition does not hold at time step $\tilde{t}_s(a)-1$ and the last term in the summation of $\widehat \Delta_{t_s:\tilde{t}_s(a)}(a_s^g, a)$ is at most 1, and the penultimate step follows from the Cauchy-Schwartz inequality.

Since we have at most $S+1$ episodes under $\cE_1 \cap \cE_2$ (by Lemma~\ref{lem:episodes}) and that $\cE_1 \cap \cE_2 \cap \cE_3$ hold with probability at least $1-3/N$ (by Lemma~\ref{lemma:hoeff} and the choice $\delta=1/N$), we obtain
\begin{align}
 \E\left[R_{N}^{(2)}\right] & \le \E\left[\one\{\cE_1 \cap \cE_2 \cap \cE_3\}
 R_{N}^{(2)}\right] + 3 \nonumber \\
 & \le \E\left[\one\{\cE_1 \cap \cE_2 \cap \cE_3\}
 \sum_s 44 C_N\sqrt{C_N^2 K^2 + K |\cN_s|} \right] + 3 \nonumber \\
& \le 44 C_N\sqrt{C_N^2 K^2 (S+1)^2 + K (S+1) N} +3\;, \label{eq:R2orig_statement}
\end{align} 
where the last step follows again by the Cauchy-Schwartz inequality.

To finish the proof, notice that if $K (S+1) \le N$, \eqref{eq:R2orig_statement} implies 
\begin{align}
\E\left[R_{N}^{(2)}\right] \le 44 C_N\sqrt{(C_N^2+1) K (S+1) N} +3\;.
\label{eq:R2final}
\end{align}
Finally, if $K (S+1) > N$, the right hand side above if lower bounded by $44 C_N^2 N > N$, hence \eqref{eq:R2final} also holds trivially in this case (as the rewards are $[0,1]$-valued, and so $R_N^{(2)} \le N$).
\end{proof}

\subsection{Bounding $R_N^{(3)}$, the regret of playing arms $a_s^g$ with respect to the optimal arms}
\label{sec:R3}

In this section, we bound the regret of playing arms $a_s^g$ with respect to playing the optimal arms. 

\begin{lemma}
\label{lem:R3}
Under the conditions of Theorem~\ref{thm:regret}, if $N\ge 2$, we have
\begin{align*}
\E[R_N^{(3)}] &\le 
C_N(24 + 50 C_N +25 C_N \sqrt{\log_2(N)+1}) \sqrt{K(2S+1)N} + 400 C_N^2\sqrt{K(S+1)N} \\
& \qquad + 800 C_N^2\sqrt{KN} + 25 K (2S+1) C_N^2 + 2 \\
&\le
\CONST\cdot C_N^2\sqrt{K(S+1)N} \log K
\end{align*}
for some appropriate universal constant $\CONST$.
\end{lemma}

The rest of the section is devoted to proving this lemma. 
We introduce a partitioning of the time horizon into several intervals, and we bound the regret of $a_s^g$ with respect to $a^*_n$ (where $n \in \cN_s$; recall that $\cN_s$ is the set of time steps in episode $s$).
Throughout we let $c_n(a)=g_n(a_n^*) - g_n(a)$ denote the instantaneous regret of an arm $a \in [K]$ in time step $n$, and $s(n)$ to denote the index of the episode $n$ belongs to (note that $s(n)$ is a random quantity).

We start with analyzing time steps where the optimal arm is in the \GOOD set in Section~\ref{sec:R3-1}, and consider the significantly more complicated case when it belongs to the \BAD set in Section~\ref{sec:R3-2}. Some technical lemmas are presented in Section~\ref{sec:lemmas}.

\paragraph{Partitioning the time horizon.} 
In the analysis we partition the time horizon as follows:
Let $\tau_1 = 1 < \tau_2 < \cdots <\tau_M \le \tau_{M+1}:=N$ denote the time steps when either a new episode starts or the identity of the optimal arm changes. Note that $M$ is random,
and since by Lemma~\ref{lem:episodes} the number of episodes is at most $S+1$ under $\cE_1 \cap \cE_2$, and the identity of the optimal arm changes $S$ times, we have $M \le 2S+1$ when $\cE_1 \cap \cE_2$ holds (i.e., with high probability).
For any $i \in [M]$, let $\tau'_i = \max \{t \in [\tau_i:\tau_{i+1}-1]: a^*_t \in \GOOD_t\}$ if the latter set is non-empty; we let $\tau'_i=\tau_i-1$ if the optimal arm is in the \BAD set during segment $[\tau_i:\tau_{i+1}-1]$. Then (i) the intervals $\{[\tau_i:\tau'_i-1],[\tau_i':\tau_{i+1}-1]: i \in [M]\}$ provide a partition of $[T]$; (ii) each episode $s$ is of the form $[\tau_{i_s}:\tau_{i_s+m_s}-1]$ where $i_s$ is the index marking the beginning of the episode and $m_s$ is the (random) number of segment $[\tau_i:\tau_{i+1}-1]$ in episode $s$; (iii) $[\tau_i:\tau'_i-1]$ and $[\tau_i':\tau_{i+1}-1]$ belong to the same (single) episode for each $i \in [M]$, and $a_{s(n)}^g$ and $a^*_n$ are constant in the time steps $n$ of any of these intervals. Furthermore, letting $\cN^G:=\{n: a^*_n \in \GOOD_n\}$ and $\cN^B=\{n: a^*_n \in \BAD_n\}$ denote the set of time steps when the optimal arm belongs to the \GOOD, respectively \BAD, we have $\cN^G = \bigcup_{i \in [M]}[\tau_i:\tau'_i-1]$ and
$\cN^B = \bigcup_{i \in [M]}[\tau'_i:\tau_{i+1}-1]$.

\subsubsection{The optimal arm belongs to the \GOOD set}
\label{sec:R3-1}

In this section we consider the regret of the arms $a_s^g$ on $\cN^G$.
Consider an interval $I=[\tau_i:\tau'_i-1]$, and let $s=s(\tau_i)$ denote the episode it belongs to. Since both $a_s^g$ and the optimal arm $a^*_I:=a^*_n$ belong to $\GOOD_n$ for all $n \in I$, Lemma~\ref{lem:activeArms}(ii) implies that if $\cE_2$ holds, 
\begin{align*}
\sum_{n \in [\tau_i:\tau'_i-1]} c_n(a_s^g) = K \widetilde \Delta_{\tau_i,\tau'_i-1}(a^*_I,a_s^g) \le 24 C_N \left(\sqrt{K(\tau'_i-\tau_i)} \vee K C_N\right)\;.
\end{align*}
Taking into account that $c_t \le 1 \le K C_N^2$, we have that under $\cE_2$,
\begin{align*}
\sum_{n \in [\tau_i:\tau'_i]} c_n(a_s^g) \le
24 C_N \sqrt{K(\tau'_i-\tau_i)} + 25 K C_N^2\;.
\end{align*}
Since the total length of the intervals is at most $N$ and the number of intervals is at most $2S+1$ under $\cE_1 \cap \cE_2$, the Cauchy-Schwartz inequality implies that the regret of the arms $a_s(n)^g$ under this event for time steps when the optimal arm is in the \GOOD set can be bounded as
\begin{align}
\one\{\cE_1 \cap \cE_2\} \sum_{n=1}^N \one\{a^*_n \in \GOOD_n\} c_n(a_{s(n)}^g)
& \le \one\{\cE_1 \cap \cE_2\}  \sum_{i=1}^m \left(24 C_N \sqrt{K(\tau'_i-\tau_i)} + 25 K C_N^2\right) \nonumber\\
& \le \one\{\cE_1 \cap \cE_2\}  \left(24 C_N \sqrt{m K N} + 25 m K C_N^2 \right) \nonumber\\
& \le 24 C_N \sqrt{K(2S+1)N} + 25 K (2S+1) C_N^2\;.
\label{eq:R3good}
\end{align}
Note that the above bound is of order $C_N^2\sqrt{K(S+1)N}$, as it was discussed in the proof of \eqref{eq:R2final} from \eqref{eq:R2orig_statement}.

\subsubsection{The optimal arms belongs to the \BAD set}
\label{sec:R3-2}

Fix an episode $s$ (recall that this episode is $[\tau_{i_s}:\tau_{i_s+m_s}-1]$). 
In what follows, we bound the regret of arm $a_s^g$ for time steps when the optimal arm belongs to the \BAD set, that is, over all intervals of the form $[\tau'_i,\tau_{i+1}-1]$ for $i\in [i_s:i_s+m_s-1]$.
During the proof we consider the shortest intervals in Lemma~\ref{lem:activeArms}(ii) which ensure that $a_s^g$ is eliminated by the optimal arm (under $\cE_2$). The regret of $a_s^g$ on these intervals is well-controlled, hence their overall contribution to $\E[R^{(3)}]$ is acceptable if there are not too many of them. On the other hand, the more such intervals are in the episode, the larger the probability of detecting the suboptimality of $a_s^g$, hence ending the episode, and limiting $a_s^g$'s contribution to the total regret. Since $a_s^g$ is determined only at the end of segment $s$, we bound the regret of every arm $a \in [K]$. 

Fix an arm $a$. Consider an index $i\in [i_s:i_s+m_s-1]$ and interval $[\tau'_i:\tau_{i+1}-1]$ in episode $s$ (that is, the part of the segment $[\tau_i:\tau_{i+1}-1]$ where the optimal arm is in the \BAD set). For a time step $n \in [\tau'_i+1,\tau_{i+1}-1]$, let $n'\in [n:\tau_{i+1}]$ be the smallest integer such that the following condition is satisfied for some $n \le n'' \le n'$:
\begin{equation}
\label{eq:candidate}
\sum_{t=n''}^{n'-1} c_t(a) > 24 C_N \left(\sqrt{K (n'-n'')} \vee K C_N \right) \;.
\end{equation}
If no such $n'$ exists, let $n'=\tau_{i+1}$.
Let $I_n:=[n:n'-1]$. Given that $c_t\in [0,1]$ and $C_N\ge 1$, the above condition implies $|I_n| = n'-n \ge n'-n'' \ge K$.
Note that if the optimal arm does not change in $I_n$, then the left-hand side of \eqref{eq:candidate} is $K \widetilde \Delta(a_{n'}^*,a)$, and \eqref{eq:candidate} is the elimination condition in Lemma~\ref{lem:activeArms}(ii) (if both $a$ and $a^*_n$ are active in $I_n$.
As such, an interval constructed in this way is called a \emph{candidate} interval for eliminating $a$. 
By definition, \eqref{eq:candidate} does not hold for the interval $[n:n'-2]$, hence, since $c_t\in [0,1]$, $C_N \ge 1$ and $|I_n|\ge K$,
\begin{equation}
\label{eq:regret-asg}
\sum_{t \in I_n} c_t(a) \le 25 C_N \left(\sqrt{K |I_n|} \vee K C_N \right) \le 25 C_N^2 \sqrt{K |I_n|} \,,
\end{equation}
which provides a bound on the regret of $a$ in $I_n$ with respect to the optimal arm (note that \eqref{eq:regret-asg} also applies when \eqref{eq:candidate} does not hold and $n'=\tau_{i+1}$).

Let $E_n$ denote the event that the optimal arm is active in interval $I_n$. By Lemma~\ref{lem:activeArms}(ii), $E_n$ leads to the elimination of $a$ when $\cE_2$ holds, hence we refer to $E_n$ as an \emph{elimination event}. If the optimal arm $a_n^*$ is in $\BAD_n$, it is active in $I_n$ if a new exploration obligation of length at least $|I_n|/K$ is prescribed at the beginning of time step $n$. This happens with probability at least $\frac{1}{2}\sqrt{\frac{s}{N|I_n|}}$  
since a new exploration obligation of length $1/\varepsilon^2$ is prescribed with probability $\varepsilon\sqrt{s/(KN)}$ (see line~\ref{ln:randomchoice} in Algorithm~\ref{alg:armswitch}) for all $\varepsilon \in \cB$.
Therefore, $\P(\overline{E_n}|a_n^* \in \BAD_n) \le 1-\frac{1}{2}\sqrt{\frac{s}{N|I_n|}}$. Define $Q_I := \frac{1}{2}\sqrt{\frac{s}{N|I|}}$ for any interval $I \subset [N]$ and notice that $Q_I \ge Q_{I'}$ if $I\subset I'$. 

Let $\cJ=\{J_1, J_2,\ldots,J_{m}\}$ be a partitioning of $[\tau'_i,\tau_{i+1}-1]$ into candidate intervals ordered by their starting points (if $i<k$ then the starting point of $J_i$ is smaller than that of $J_k$); note that $m$ is a random quantity which depends on $\tau'_i$, the starting point of the segment. 
By the construction of candidate intervals, if an elimination event $E_n$ happens in an interval $J_k$ (i.e., $n \in J_k$), it finishes by the end of $J_{k+1}$, and if $\cE_2$ holds, it means that $a$ is eliminated from the \GOOD set by then (when $a=a_s^g$, this means that the episode also finishes); it is easy to see that this is also true if $I_n$ ends at the end of the episode, that is, in time step $\tau_{i+1}-1$, since in this case either $J_k$ or $J_{k+1}$ ends at the same time.

Let $F_k=\bigcup_{n\in J_k} E_n$ denote the event that an elimination event happens in $J_k$, and $\overline{F}^k:=\bigcap_{j=1}^k \overline{F_j}$ that no elimination event happens before $J_k$. By the argument above, if $\cE_2$ holds, then $\overline{F}^{k-1} \cap F_k$ means that $a$ is eliminated from the \GOOD set by an elimination event happening in $J_k$, hence $a$ can only be in the \GOOD set until the end of $J_{k+1}$.  Therefore, letting 
\[
R^G(a,[\tau'_i:\tau_{i+1}-1]) = \sum_{n=\tau'_i}^{\tau_{i+1}-1}
\one\{a \in \GOOD_n\} c_n(a)
\]
denote the regret of arm $a$ in interval $[\tau'_i+1:\tau_{i+1}-1]$ when it is in the \GOOD set, we have
\begin{align*}
\one\{\cE_1 \cap \cE_2\}\frac{R^G(a,[\tau'_i:\tau_{i+1}-1])}{25 C_N^2} &\le \one\{\cE_1 \cap \cE_2 \cap F_1\} (\sqrt{K |J_1|}+\sqrt{K |J_2|})\nonumber \\ 
&\qquad+ \one\{\cE_1 \cap \cE_2 \cap \overline{F_1}\cap F_2\} (\sqrt{K |J_1|}+\sqrt{K |J_2|}+\sqrt{K |J_3|})\nonumber \\ 
&\qquad+ \ldots \nonumber \\
&\qquad+ \one\{\cE_1 \cap \cE_2 \cap \overline{F_1}\cap \ldots \cap \overline{F_{m-2}} \cap F_{m-1}\} (\sqrt{K |J_1|}+\ldots +\sqrt{K |J_m|})\nonumber \\
&\qquad+ \one\{\cE_1 \cap \cE_2 \cap \overline{F_1}\cap \ldots \cap \overline{F_{m-2}} \cap \overline{F_{m-1}}\} (\sqrt{K |J_1|}+\ldots +\sqrt{K |J_m|}) \nonumber \\
&\le \one\{\cE_1 \cap \cE_2 \cap \overline{F_1}\cap F_2\} \sqrt{K |J_1|}\nonumber \\ 
&\qquad+ \ldots \nonumber \\
&\qquad+ \one\{\cE_1 \cap \cE_2 \cap \overline{F_1}\cap \ldots \cap \overline{F_{m-2}} \cap F_{m-1}\} (\sqrt{K |J_1|}+\ldots +\sqrt{K |J_{m-2}|})\nonumber \\
&\qquad+ \one\{\cE_1 \cap \cE_2 \cap \overline{F_1}\cap \ldots \cap \overline{F_{m-2}} \cap \overline{F_{m-1}}\} (\sqrt{K |J_1|}+\ldots +\sqrt{K |J_{m-2}|}) \nonumber \\
&\qquad+2\one\{\cE_1 \cap \cE_2\}\sqrt{K(\tau_{i+1}-\tau_i)} \nonumber \\
&= \one\{\cE_1 \cap \cE_2 \cap \overline{F}^1\} \sqrt{K |J_1|}
+ \one\{\cE_1 \cap \cE_2 \cap \overline{F}^2\} \sqrt{K |J_2|} \nonumber \\
& \qquad + \ldots
+ \one\{\cE_1 \cap \cE_2 \cap \overline{F}^{m-2}\} \sqrt{K |J_{m-2}|}
+2\one\{\cE_1 \cap \cE_2\}\sqrt{K(\tau_{i+1}-\tau_i)}\;,
\end{align*} 
where (i) the first inequality holds by the above argument about the elimination of $a$ form the \GOOD set and by \eqref{eq:regret-asg} bounding its regret in intervals $J_1,J_2,\ldots$;
(ii) the second inequality holds by bounding the regret of the last two intervals in every row by $\sqrt{K(\tau_{i+1}-\tau_i)}$, as obviously $|J_k| \le \tau_{i+1}-\tau_i$ for all $k$; (iii) the last equality follows by collecting like terms. Therefore,
\begin{align}
\label{eq:Rtau'-tau}
\one\{\cE_1 \cap \cE_2\} R^G(a,[\tau'_i:\tau_{i+1}\!-\!1]) &\le
25 C_N^2\left( 2\one\{\cE_1 \cap \cE_2\}\sqrt{K(\tau_{i+1}-\tau_i)} + \sum_{k=1}^{m-2} \one\big\{\cE_1 \cap \cE_2 \cap \overline{F}^k\big\} \sqrt{K |J_k|}\right).
\end{align}

To proceed from here, we need to bound the probability of events $\overline{F}^k$ for $k \in [m-2]$.
Let $\{J_1^k, J_2^k,\ldots, J_{|J_k|}^k\}$ be the set of all candidate intervals starting in $J_k$. Note that by definition, all intervals $J_i^k$ end no later than the last time step of $J_{k+1}$, hence $J_i^k \subset J_k \cup J_{k+1}$.
Therefore, conditioned on $\tau'_i$, the probability that no elimination event happens before the end of $J_k$ can be upper bounded as\footnote{Note that history $\cH'_n$ is defined in line~\ref{ln:new_step} of Algorithm~\ref{alg:armswitch}.}
\begin{align}
\notag
\P\left(\overline{F}^k\middle| \cH'_{\tau'_i} \right) & = \P\Bigg(\bigcap_{n \in \cup_{j=1}^k J_j} \overline{E_n} 
\Big| \cH'_{\tau'_i} \Bigg) \\
&\le \prod_{n \in \cup_{j=1}^k J_j} \P(\overline{E_n}) \notag \\
&\le (1-Q_{J_1^1})\cdot (1-Q_{J_2^1}) \ldots (1-Q_{J_{|J_1|}^1}) \cdot (1-Q_{J_1^2})\cdot (1-Q_{J_2^2})\ldots (1-Q_{J_{|J_2|}^2}) \notag \\ 
\notag
&\qquad\qquad\cdot (1-Q_{J_1^k})\cdot (1-Q_{J_2^k}) \ldots (1-Q_{J_{|J_k|}^k}) \\ 
\notag
&\le (1-Q_{J_1\cup J_2})\cdot (1-Q_{J_1\cup J_2}) \ldots (1-Q_{J_1\cup J_2}) \cdot (1-Q_{J_2\cup J_3})\cdot (1-Q_{J_2\cup J_3})\ldots (1-Q_{J_2\cup J_3}) \\ 
\notag
&\qquad\qquad\cdot (1-Q_{J_k\cup J_{k+1}})\cdot (1-Q_{J_k\cup J_{k+1}}) \ldots (1-Q_{J_k\cup J_{k+1}}) \\
\notag
&= \left(1-Q_{J_1\cup J_2}\right)^{|J_1|} \cdot \left(1 -  Q_{J_2\cup J_3}\right)^{|J_2|} \ldots (1-Q_{J_k\cup J_{k+1}})^{|J_k|}\\
&\le \exp\left(- |J_1| Q_{J_1\cup J_2} \right) \cdot \exp\left(- |J_2| Q_{J_2\cup J_3} \right)\ldots \exp\left(- |J_k| Q_{J_k\cup J_{k+1}} \right) \notag \\
& =
\exp\left(-\frac{1}{2}\sqrt{\frac{s}{N}} \sum_{j=1}^k\frac{|J_j|}{\sqrt{|J_j|+|J_{j+1}|}}\right),
\label{eq:elim-prob}
\end{align}
where in the third inequality, we used the fact that $Q_I \ge Q_{I\cup I'}$ for any intervals $I,I'$, and in the last inequality we used that $1-x \le e^{-x}$ for any real number $x$. 

We further bound \eqref{eq:elim-prob} by Lemma~\ref{lem:sumsq2}. To do so, let $\cJ^d=\{J_k \subset \cJ: |J_{k+1}| > \sum_{j=1}^k |J_j|\}$ denote the set of intervals where adding $J_{k+1}$ at least doubles the length of the interval $\bigcup_{j=1}^k J_k$ covered so far (these are the intervals where the second part of the lemma does not apply). 
We bound the regret of $a$ on intervals in $\cJ^d$ separately. Because of the aforementioned doubling of length, and since $\left|\bigcup_{J \in \cJ} J\right| \le \tau_{i+1}-\tau'_i$, $|\cJ^d| \le \log_2(N)+1$. Therefore, by the Cauchy-Schwartz inequality,
\begin{align}
\label{eq:regret-Jd}
\sum_{J \in \cJ^d} 25 C_N^2\sqrt{K|J|}
\le 25 C_N^2\sqrt{K(\tau_{i+1}-\tau'_{i})(\log_2(N)+1)} \;.
\end{align}

To consider the remaining intervals, let $\cJ'=\cJ \setminus \cJ^d$, $m'=|\cJ'|$, and let $v:[m] \to [m'] \cup \{\star\}$ be an index mapping such that $v(k)$ is the index of $J_k$ in $\cJ'$ according to ordering by the starting point if $J_k \in \cJ'$ and $\star$ otherwise, and let $v^{-1}$ define its inverse mapping restricted to $[m']$. Furthermore, let $x_k:=|J_{v^{-1}(k)}|$. Then, by \eqref{eq:elim-prob} and Lemma~\ref{lem:sumsq2}, for any $k$ such that $J_k \in \cJ'$,
\begin{align}
\P\left(\overline{F}^k\middle| \cH'_{\tau'_i} \right)    
& \le \exp\left(-\frac{1}{2}\sqrt{\frac{s}{N}} \sum_{j=1}^k\frac{|J_j|}{\sqrt{|J_j|+|J_{j+1}|}}\right) \nonumber \\
&\le \exp\left(-\frac{1}{8}\sqrt{\frac{s}{N}} \sum_{j=1}^k\sqrt{|J_j|}\right)\nonumber \\
& = \exp\left(-\frac{1}{8}\sqrt{\frac{s}{N}} \sum_{j \in [1:k]: J_j \in \cJ'}\sqrt{|J_j|}\right) \nonumber \\
& \le \exp\left(-\frac{1}{8}\sqrt{\frac{s}{N}} \sum_{j=1}^{m'} \sqrt{x_j}\right)~.
\label{eq:PFkbound}
\end{align}

Now let $R^G(a,\tau_i):=\sum_{n \in [\tau_i:\tau_{i_s+m_s}-1]} \one\{a \in \GOOD_n, a_n^*\in\BAD_n\} c_n(a)$ denote the regret of arm $a$ when the optimal arm is in \BAD and $a$ is \GOOD in episode $s$ starting from $\tau_i$. Let $D_N=25 C_N^2\sqrt{K}$ and $D'_N=25 C_N^2\sqrt{K}(2+\sqrt{\log_2(N)+1})$, and we also use the notation $\overline{F}^{J_k}:=\overline{F}^k$ and $\cJ''=\cJ'\setminus \{J_{m-1},J_m\}$. We bound $\E\left[\one\{\cE_1 \cap \cE_2\}R^G(a,\tau_i)|\cH'_{\tau_i}\right]$ recursively as follows:
\begin{align}
\lefteqn{\E\left[\one\{\cE_1 \cap \cE_2\}R^G(a,\tau_i)|\cH'_{\tau_i}\right]} \nonumber \\ 
&= \E\Bigg[\E\Big[\one\{\cE_1 \cap \cE_2\} R^G(a,\tau_i)\,\Big|\,\cH'_{\tau'_i}\Big]\,\Bigg|\,\cH'_{\tau_i}\Bigg] \nonumber \\
&= \E\Bigg[\E\bigg[\one\{\cE_1 \cap \cE_2\} R^G(a,[\tau'_i:\tau_{i+1}-1]) + \one\big\{\cE_1 \cap \cE_2 \cap\overline{F}^{m-2}\big\} R^G(a,\tau_{i+1}) \,\bigg|\,\cH'_{\tau'_i}\bigg]\,\Bigg|\,\cH'_{\tau_i}\Bigg] \nonumber \\
&\le \E\Bigg[\E\bigg[D'_N\one\{\cE_1 \cap \cE_2\}\sqrt{\tau_{i+1}-\tau_i} + D_N \sum_{J \in \cJ''}  \one\big\{\cE_1 \cap \cE_2 \cap\overline{F}^{J}\big\} \sqrt{|J|} + \one\big\{\cE_1 \cap \cE_2 \cap\overline{F}^{m-2}\big\} R^G(a,\tau_{i+1})\,\bigg|\,\cH'_{\tau'_i} \bigg]\,\Bigg|\,\cH'_{\tau_i} \Bigg] \nonumber \\
&= D'_N\E\left[\one\{\cE_1 \cap \cE_2\}\sqrt{\tau_{i+1}-\tau_i}\, \middle|\cH'_{\tau_i}\right] \nonumber \\
& \qquad + \E\Bigg[\E\bigg[\E\Big[ D_N \sum_{J \in \cJ''}  \one\big\{\cE_1 \cap \cE_2 \cap\overline{F}^{J}\big\} \sqrt{|J|} + \one\big\{\cE_1 \cap \cE_2 \cap\overline{F}^{m-2}\big\} R^G(a,\tau_{i+1})\,\Big| \cH'_{\tau_{i+1}}\Big]\,\bigg|\,\cH'_{\tau'_i} \bigg]\,\Bigg|\,\cH'_{\tau_i} \Bigg] \nonumber \\
&\le D'_N\E\left[\one\{\cE_1 \cap \cE_2\}\sqrt{\tau_{i+1}-\tau_i}\, \middle|\cH'_{\tau_i}\right]\nonumber \\
& \qquad
+ \E\Bigg[\E\bigg[D_N \sum_{J \in \cJ''}  \one\big\{\overline{F}^{J}\big\} \sqrt{|J|} + \one\big\{\overline{F}^{m-2}\big\} \E\Big[\one\big\{\cE_1 \cap \cE_2\big\}R^G(a,\tau_{i+1})\,\Big| \cH'_{\tau_{i+1}}\Big]\,\bigg|\,\cH'_{\tau'_i} \bigg]\,\Bigg|\,\cH'_{\tau_i} \Bigg]\;, 
\label{eq:RGbound}
\end{align}
where (i) the second equality follows since the $a\in \GOOD_{\tau_{i+1}}$ is only possible (under $\cE_2$) if $\cE_2 \cap \overline{F}^{m-2}$ holds; (ii) the inequality holds by \eqref{eq:Rtau'-tau} and \eqref{eq:regret-Jd}; and (iii) the last inequality holds by dropping some of the indicators $\one\{\cE_1 \cap \cE_2\}$.

Let $\cS_i$ denote the collection of segments in episode $s$ after $\tau_i$, that is, $\cS_i=\{[\tau_j:\tau_{j+1}-1]:j\in [i:i_s+m_s-1]\}$. We prove by induction that
\begin{align}
\label{eq:R3-2-episode}
\E\Big[\one\{\cE_1 \cap \cE_2\}R^G(a,\tau_{i}) \,\Big|\,\cH'_{\tau_{i}}\Big] \le D'_N\E\left[\one\{\cE_1 \cap \cE_2\}\sum_{I \in \cS_i} \sqrt{|I|} \middle| \cH'_{\tau_i}\right] + 8 D_N \sqrt{\frac{ N}{s}}\;,
\end{align}
from which our desired regret bound follows easily.
Indeed, since $a_s^g$ is in the \GOOD set in the entire segment~$s$, 
\[
R^G(a_s^g,\tau_{i_s}) = \sum_{i=i_s}^{i_s+m_s-1}
\sum_{n \in [\tau'_i:\tau_{i+1}-1]} c_n(a_s^g)
= \sum_{n \in [\tau_{i_s}:\tau_{i_{s+1}-1}]}
\one\{a^*_n \in \BAD_n\} c_n(a_{s(n)}^g)\;,
\]
and so
\begin{align*}
\lefteqn{\E\left[\one\{\cE_1 \cap \cE_2\} \sum_{n=1}^N \one\{a^*_n \in \BAD_n\} c_n(a_{s(n)}^g)\right]} \\
& \le
D'_N \E\left[\one\{\cE_1 \cap \cE_2\} \sum_s \sum_{I \in \cS_{i_s}}\sqrt{|I|}\right]
+ 8 D_N \E\left[\sum_s \sqrt{\frac{ N}{s}}\right] \\
& = D'_N\E\left[\one\{\cE_1 \cap \cE_2\} \sum_i \sqrt{\tau_{i+1}-\tau_i}\right]
+ 8 D_N \sqrt{N} \E\left[\one\{\cE_1 \cap \cE_2\}\sum_s \frac{1}{\sqrt{s}}\right]
+ 8 D_N \sqrt{N} \E\left[\one\{\overline{\cE_1} \cup \overline{\cE_2}\}\sum_s \frac{1}{\sqrt{s}}\right] \\
& \le
D'_N \sqrt{(2S+1)N} + 8 D_N \sqrt{N} \sum_{s=1}^{S+1} \frac{1}{\sqrt{s}} + 8 D_N \sqrt{N}\sum_{s=1}^{N} \frac{1}{\sqrt{s}} \P(\overline{\cE_1} \cup \overline{\cE_2}) \\
& \le D'_N \sqrt{(2S+1)N} + 16 D_N \sqrt{(S+1)N}
+ 32 D_N \sqrt{N} \\
& = 25 C_N^2(2+\sqrt{\log_2(N)+1})\sqrt{K(2S+1)N}
+ 400 C_N^2\sqrt{K(S+1)N} + 800 C_N^2\sqrt{KN}
\;,
\end{align*}
where (i) the second inequality holds by the Cauchy-Schwartz inequality because $\sum_i (\tau_{i+1}-\tau_i)=N$ and under $\cE_1 \cap \cE_2$ there are at most $2S+1$ segments (as discussed at the beginning of Section~\ref{sec:R3}) and $S+1$ episodes, and without any condition there are at most $N$ episodes; (ii) the last inequality holds because
$\sum_{s=1}^n 1/\sqrt{s} \le \int_0^n 1/\sqrt{s} ds = 2\sqrt{s}$ and $\P(\overline{\cE_1} \cup \overline{\cE_2}) \le 2/N$ by Lemma~\ref{lemma:hoeff} and the choice $\delta=1/N$.
Combining the above inequality with \eqref{eq:R3good} and 
Lemma~\ref{lemma:hoeff} to show that when $\cE_1 \cap \cE_2$ does not hold, the expected regret of the algorithm is at most 2, proves Lemma~\ref{lem:R3}.

It remains to prove \eqref{eq:R3-2-episode}. By definition, $R^G(a,\tau_k)=0$ when $a$ is eliminated from the \GOOD set. This happens at latest at the end of the episode, so $R^G(a,\tau_{i_{s+1}})=0$, satisfying \eqref{eq:R3-2-episode} for $i=i_{s+1}$ (note that we can define, without loss of generality $R^G(a,N+1)=0$), so the starting assumption holds for backwards induction.
Assume now that
\[
\E\Big[\one\{\cE_1 \cap \cE_2\}R^G(a,\tau_{i+1}) \,\Big|\,\cH'_{\tau_{i+1}}\Big] \le D'_N\E\left[\one\{\cE_1 \cap \cE_2\}\sum_{I \in \cS_{i+1}} \sqrt{|I|} \middle| \cH'_{\tau_{i+1}}\right] + 8 D_N \sqrt{\frac{ N}{s}}
\]
holds. Combining this assumption with \eqref{eq:RGbound}, it remains to prove that
\begin{align*}
\E\Bigg[\E\bigg[D_N \sum_{J \in \cJ''}  \one\big\{\overline{F}^{J}\big\} \sqrt{|J|} + 8 D_N \one\big\{\overline{F}^{m-2}\big\} \sqrt{\frac{N}{s}} \,\bigg|\,\cH'_{\tau'_i} \bigg]\,\Bigg|\,\cH'_{\tau_i} \Bigg]
\le 8 D_N \sqrt{\frac{N}{s}}\;. 
\end{align*}
Let $m''=|\cJ''|$. Using the introduced $x_k$ notation for the length of the intervals as well as the upper bound \eqref{eq:PFkbound} on $\P(\overline{F}^K|\cH'_{\tau'_i})$, we have
\begin{align*}
\lefteqn{\E\Bigg[D_N \sum_{J \in \cJ''}  \one\big\{\overline{F}^{J}\big\} \sqrt{|J|} + 8 D_N \one\big\{\overline{F}^{m-2}\big\} \sqrt{\frac{N}{s}} \,\Bigg|\,\cH'_{\tau'_i} \Bigg]} \\
& \le
D_N \sum_{k=1}^{m''} \exp\left(-\frac{1}{8}\sqrt{\frac{s}{N}}\,\sum_{j=1}^{k} \sqrt{x_j} \right) \sqrt{x_k} + D_N \exp\left(-\frac{1}{8}\sqrt{\frac{s}{N}}\,\sum_{j=1}^{m''} \sqrt{x_j} \right)8 \sqrt{\frac{N}{s}} \\
& \le 8 D_N \sqrt{\frac{N}{s}}
\end{align*}
by Lemma~\ref{lem:expalpha} with $\alpha=\sqrt{s/N}/8$ and $y_j=\sqrt{x_j}$.
This finishes the proof of Lemma~\ref{lem:R3}.

\subsection{Technical lemmas}
\label{sec:lemmas}

\begin{lemma}
\label{lem:sumsq2}
Let $x_1,\ldots, x_{n+1}$ be a sequence of positive reals. Then
\[
\frac{1}{2} \sum_{j=1}^n \sqrt{x_j} - \frac{\sqrt{x_{n+1}}}{4} \le \sum_{j=1}^n \frac{x_j}{\sqrt{x_j+x_{j+1}}} \;.
\]
In particular,
if $x_{n+1} \le \sum_{j=1}^n x_j$, then
\[
\frac{1}{4} \sum_{j=1}^n \sqrt{x_j} \le \sum_{j=1}^n \frac{x_j}{\sqrt{x_j+x_{j+1}}} \;.
\]
\end{lemma}
\begin{proof}
We have
\begin{align}
\sum_{j=1}^n \frac{x_j}{\sqrt{x_j+x_{j+1}}}
& \ge
\sum_{j=1}^n \frac{x_j}{\sqrt{x_j}+\sqrt{x_{j+1}}}
\label{eq:lll1} \\
& = \sum_{j=1}^n \left(\frac{x_j-x_{j+1}}{\sqrt{x_j}+\sqrt{x_{j+1}}} +
\frac{x_{j+1}}{\sqrt{x_j}+\sqrt{x_{j+1}}} \right) \nonumber\\
& = \sum_{j=1}^n \left( \sqrt{x_j} - \sqrt{x_{j+1}}
+ \frac{x_{j+1}}{\sqrt{x_j}+\sqrt{x_{j+1}}} \right) \nonumber\\
&= \sqrt{x_1} - \sqrt{x_{n+1}} +
\sum_{j=1}^n \frac{x_{j+1}}{\sqrt{x_j}+\sqrt{x_{j+1}}}~.
\label{eq:lll2}
\end{align}
Combining \eqref{eq:lll1} and \eqref{eq:lll2}, we obtain
\begin{align*}
 \sum_{j=1}^n \frac{x_j}{\sqrt{x_j+x_{j+1}}}
& \ge \frac{\sqrt{x_1}-\sqrt{x_{n+1}}}{2}
+ \frac{1}{2} \sum_{j=1}^n \frac{x_j + x_{j+1}}{\sqrt{x_j}+\sqrt{x_{j+1}}} \\
& \ge \frac{\sqrt{x_1}-\sqrt{x_{n+1}}}{2}
+ \frac{1}{4} \sum_{j=1}^n \left(\sqrt{x_j} + \sqrt{x_{j+1}}\right) \\
& \ge \frac{1}{2} \sum_{j=1}^n \sqrt{x_j} - \frac{\sqrt{x_{n+1}}}{4}\;,
\end{align*}
finishing the proof of the first statement.
The second statement 
holds because $x_{n+1} \le \sum_{j=1}^n x_j$ implies
$\sqrt{x_{n+1}} \le \sqrt{\sum_{j=1}^n x_j} \le \sum_{j=1}^n \sqrt{x_j}$.
\end{proof}

\begin{lemma}
\label{lem:expalpha}
Let $\alpha,y_1,\ldots,y_m$ be positive real numbers.
Then
\begin{equation}
\label{eq:recursion}
\sum_{k\in [m]} \exp\left(-\alpha\sum_{j=1}^{k} y_j \right) y_k + \exp\left(-\alpha\sum_{j=1}^{m} y_j \right)\frac{1}{\alpha} \le \frac{1}{\alpha} \;.
\end{equation}
\end{lemma}
\begin{proof}
First, we show that $e^{-\alpha y} y + e^{-\alpha y}\frac{1}{\alpha} \le \frac{1}{\alpha}$. Let $f(y)=e^{-\alpha y} y + e^{-\alpha y}\frac{1}{\alpha}$, and observe that 
\[
f'(y) = e^{-\alpha y} -\alpha e^{-\alpha y} y - e^{-\alpha y} = -\alpha e^{-\alpha y} y \le 0 \;.
\]
Therefore $f$ is maximized at $y=0$, and so $f(y) \le f(0) = \frac{1}{\alpha}$. 
We finish the proof by induction. Assume \eqref{eq:recursion} holds for $m-1$ (in place of $m$). Then
\begin{align*}
\lefteqn{\sum_{k\in [m]} \exp\left(-\alpha\sum_{j=1}^{k} y_j \right) y_k + \exp\left(-\alpha\sum_{j=1}^{m} y_j \right)\frac{1}{\alpha} \le \frac{1}{\alpha} } \\
& =
 \sum_{k\in [m-1]} \exp\left(-\alpha\sum_{j=1}^{k} y_j \right) y_k +
\exp\left(-\alpha\sum_{j=1}^{m-1} y_j \right) \left(e^{-\alpha y_m}y_m + e^{-\alpha y_m}\frac{1}{\alpha}\right)\\
&\le \sum_{k\in [m-1]} \exp\left(-\alpha\sum_{j=1}^{k} y_j \right) y_k + \exp\left(-\alpha\sum_{j=1}^{m-1} y_j \right) \frac{1}{\alpha} \\
& \le \frac{1}{\alpha}\;,
\end{align*}
where in the last step we used the induction hypothesis.
This completes the proof.
\end{proof}

\section{Conclusions}

We have introduced \algoname, an algorithm for learning in non-stationary stochastic multi-armed bandit environments. The main feature of our algorithm is that its regret scales as $\widetilde O(\sqrt{KSN})$, where $S$ is the number of changes in the identity of the optimal arm. In contrast, existing work for this problem bounds the regret in terms of the number of changes in the reward function, which can be much larger. Our algorithm is based on the \textsc{AdSwitch} algorithm of \citet{auer2018adaptively,AGO-2019}, however, with substantially modified exploration and episode-reset rules.

\bibliography{refs}

\appendix

\section{Auxiliary proofs}
\label{sec:proof_lemmas}

We start with a Freedman-style martingale tail inequality:
\begin{theorem}[\citet{AHKLL-2014}]
\label{thm:freedman}
Let $(\mathcal{H}_t; t\ge 1)$ be a filtration, $(X_t; t\ge 1)$ be a real-valued martingale difference sequence adapted to $(\mathcal{H}_t)$ (i.e., $\E[X_t|\cH_t]=0$). If $|X_t|\le B$ almost surely, then for any $\eta\in (0,1/B]$, with probability at least $1-\delta$,
\[
\sum_{t=1}^n X_t \le \eta (e-2) \sum_{t=1}^n \E[X_t^2|\mathcal{H}_{t}] + \frac{\log(1/\delta)}{\eta} \;.
\]
\end{theorem}
The optimal $\eta$ minimizing the above bound is
$\eta^*=\sqrt{\frac{\log(1/\delta)}{(e-2)\sum_{t=1}^n \E[X_t^2|\mathcal{H}_{t-1}]}}$, which would lead to a bound
\[
\sum_{t=1}^n X_t \le 2\sqrt{(e-2) \log(1/\delta)\sum_{t=1}^n \E[X_t^2|\mathcal{H}_{t}]}\;.
\]
However, this cannot be achieved since it requires the knowledge of the sum of $\E[X_t^2|\mathcal{H}_{t}]$ to tune $\eta$, and it is not guaranteed that $\eta^* \le 1/B$. The next corollary takes care of these problems.

\begin{corollary}
\label{cor:freedman}
Under the conditions of Theorem~\ref{thm:freedman}, for any $0<\delta \le 2/e$,
\[
\left|\sum_{t=1}^n X_t \right|\le 2e \sqrt{(e-2) \log\left(\frac{\log(n/B)+2}{\delta}\right) \sum_{t=1}^n \E[X_t^2|\mathcal{H}_{t}]}
\; \vee \;
2 B \log\left(\frac{\log(n/B)+2}{\delta}\right)
\]
holds with probability at least $1-\delta$.
\end{corollary}
\newcommand{\grid}{\mathcal{G}}
\newcommand{\etamin}{\eta_{\min}}
\begin{proof}
We first prove the bound on $\sum_{t=1}^n X_t$ with $\delta'=\delta/2$, then we apply the same bound to $\sum_{t=1}^n (-X_t)$.
We create an exponential grid for $\eta$; applying Theorem~\ref{thm:freedman} for each value in the grid together with the union bound will prove the corollary.
Since $|X_t|\le B$, the smallest possible value of $\eta^*$ is $\etamin=\sqrt{\frac{\log(1/\delta')}{(e-2)nB}}$.
Assume first that $\etamin\le 1/B$. 
Let $\grid=\{e^{-i}/B: i \in \lfloor \log(n/B)/2\rfloor\}$.
Then, by Theorem~\ref{thm:freedman} and the union bound, with probability at least $1-\delta'$, we have
\begin{equation}
\label{eq:freedman1}
\sum_{t=1}^n X_t \le \eta (e-2) \sum_{t=1}^n \E[X_t^2|\cH_{t}] + \frac{\log(\log(n/B)/2+1)/\delta')}{\eta}
\end{equation}
for all $\eta \in \grid$ simultaneously.

The smallest element of $\grid$, denoted $\etamin'$, satisfies
\[
\etamin'=e^{-\lfloor \log(n/B)/2\rfloor}/B \le e /\sqrt{Bn} \le 
e \sqrt{\frac{\log(1/\delta')}{(e-2)nB}} = e \etamin
\]
where the second inequality holds because $\delta'\le 1/e$.
Therefore, for any $\eta \in [\etamin,1/B]$ there is an $\eta' \in \grid$ such that $\eta\le \eta' \le \eta e$.
Thus, by \eqref{eq:freedman1}, for any $\eta \in [\etamin,1/B]$,
\begin{align*}
\sum_{t=1}^n X_t 
& \le \eta' (e-2) \sum_{t=1}^n \E[X_t^2|\cH_{t-1}] + \frac{\log(\log(n/B)/2+1)/\delta')}{\eta'} \nonumber \\
& \le e \left[\eta (e-2) \sum_{t=1}^n \E[X_t^2|\cH_{t}] + \frac{\log(\log(n/B)/2+1)/\delta')}{\eta}\right]\;.
\end{align*}
Specifically, for the minimizer $\eta_m=\sqrt{\frac{\log(\log(n/B)/2+1)/\delta')}{(e-2)\sum_{t=1}^n \E[X_t^2|\cH_{t}]}}$ of the above bound, we obtain
\begin{equation}
\label{eq:freedman3}
  \sum_{t=1}^n X_t  
  \le 2e \sqrt{(e-2) \log\left(\frac{\log(n/B)/2+1}{\delta'}\right) \sum_{t=1}^n \E[X_t^2|\cH_{t}]}
\end{equation}
If $\eta_m \le 1/B$, the above bound holds. If not,
\[
\sum_{t=1}^n \E[X_t^2|\cH_{t}] \le B^2 \frac{\log(\log(n/B)/2+1)/\delta')}{e-2}
\]
and hence by \eqref{eq:freedman1} for $\eta=1/B \in \grid$, we get
\[
\sum_{t=1}^n X_t 
\le 2 B \log\left(\frac{\log(n/B)/2+1}{\delta'}\right)\;.
\]
Since one of the last two inequalities always hold, we obtain
that with probability at least $1-\delta'$,
\[
\sum_{t=1}^n X_t 
\le 2e \sqrt{(e-2) \log\left(\frac{\log(n/B)/2+1}{\delta'}\right) \sum_{t=1}^n \E[X_t^2|\cH_{t}]}
  \; \vee \;
2 B \log\left(\frac{\log(n/B)/2+1}{\delta'}\right)\;.
\]
Using the same bound for $(-X_t)$, and using the union bound proves the corollary.
\end{proof}

Now we are ready to prove Lemma~\ref{lemma:hoeff}.

\medskip

\noindent {\bf Proof of Lemma~\ref{lemma:hoeff}.}
\begin{proof}
Proof for $\cE_1$: Fix a time $n\in [N]$ and an arm $a\in [K]$. Define variable $X_t =  \one\{A_t=a\} r_t - P_t(a)g_t(a)$. Then $|X_t| \le 1$, $\E[X_t|\cH_t] = 0$, and
\begin{align*}
\E[X_t^2|\cH_t, A_t] 
& = 
\E\left[
\Big(\one\{A_t=a\} \big(r_t - g_t(a)\big) + \big(\one\{A_t=a\} - P_t(a)\big)g_t(a)\Big)^2 \big| \cH_t, A_t
\right] \\
& = \one\{A_t=a\} \E\left[\big(r_t - g_t(a)\big)^2|\cH_t,A_t\right]
+ \big(\one\{A_t=a\} - P_t(a)\big)^2g_t^2(a).
\end{align*}
Therefore, since $r_t$ takes values in $[0,1]$, $\text{Var}[r_t|A_t,\cH_t] \le 1/4$,
\begin{align*}
\E[X_t^2|\cH_t] & \le P_t(a)/4 + P_t(a)(1-P_t(a)) g_t^2(a)
\le 5P_t(a)/4.
\end{align*} 
Therefore, with probability at least $1-\delta'$,
\begin{align*}
\left| \sum_{t=n'}^n X_t \right|
& \le 2e\sqrt{\frac{5}{4}(e-2)\log\left(\frac{\log(n-n'+1)+2}{\delta'}\right) P_{n':n}(a)} \vee 2 \log\left(\frac{\log(n-n'+1)+2}{\delta'}\right) \\
& \le 6\sqrt{\log\left(\frac{\log(n-n'+1)+2}{\delta'}\right) P_{n':n}(a)} \vee 2 \log\left(\frac{\log(n-n'+1)+2}{\delta'}\right)
\end{align*}
Using $\delta'=\delta/(2K N^2)$ and the union bound over all intervals $[n':n]$ and arms $a$ (note that this choice of $\delta'$ satisfies $\delta' \le 2/e$ for all $\delta \in (0,1)$), we
obtain that with probability at least $1-\delta$
\begin{align*}
\left| \widehat{G}_{n:n'}(a) - G_{n:n'}(a)\right|
&\le 6 C_{n,n'} (\sqrt{P_{n:n'}(a)} \vee C_{n,n'}),
\end{align*}
finishing the proof of the bound for $\cE_1$.

For $\cE_2$, letting $X_t=\one\{\widetilde A_t=a\} r_t - P_t(a)g_t(a)$ yields the desired bound similarly, because here one can show that $\E[X_t^2|\cH_t] \le 5/(4K)$.

For $\cE_3$, let $X_t = (\one\{A_t=a\} - P_t(a)) g_t(a')$. Then
$\E[X_t^2|\cH_t] \le P_t$, and we obtain the bound
\begin{align*}
\left| G'_{n:n'}(a,a') - G_{n:n'}(a,a')\right|
&\le 5 C'_{n,n'} (\sqrt{P_{n:n'}(a)} \vee C'_{n,n'}),
\end{align*}
where we introduced $C'_{n,n'}$ because here the union must be taken over all ordered pairs of arms $a,a' \in [K]$ instead of all arms.

Finally, for $\cE_4$ introducing $X_t=\big(g_t(a)-g_t(a')\big)\big(\one\{\widetilde A_t=a\} - 1/K\big)$ gives $\E[X_t^2|\cH_t] \le 1/K$, which yields the desired bound.

\end{proof}

\noindent {\bf Proof of Lemma~\ref{lem:activeArms}.}
\begin{proof}
If $\Delta_{n':n}(a',a) > 24 C_{n',n} \left(\sqrt{P_{n':n}} \vee C_{n',n} \right)$ for $a,a'\in\GOOD_n$, then by the definition of $\cE_1$ and by \eqref{eq:eqpDelta}, $\better{a'}{a}{n}$ is true:
\begin{align*}
\widehat \Delta_{n':n}(a',a) &\ge \Delta_{n':n}(a',a) - 12 C_{n',n} \left(\sqrt{P_{n':n}(a)} \vee C_{n',n} \right)
 > 12 C_{n',n} \left(\sqrt{P_{n':n}(a)} \vee C_{n',n} \right),
\end{align*}
proving the first part of (i).
If $a,a' \in \GOOD_{n+1}$, then $a$ is not eliminated at the end of time step $n$, hence $\better{a'}{a}{n}$ is false. Consequently, $\Delta_{n':n}(a',a) \le 24 C_{n',n} \left(\sqrt{P_{n':n}} \vee C_{n',n} \right)$, finishing the proof of part (i).
Part (ii) can be shown similarly.

\end{proof}

\end{document}